\documentclass{article}

\usepackage{geometry}

\newgeometry{vmargin={15mm}, hmargin={35mm,35mm}}   % set the margins

\usepackage{booktabs} % For formal tables
\usepackage{amsmath}
\usepackage{amssymb}
\usepackage{multirow}
\usepackage{multicol}
\usepackage{soul}
\usepackage{url}            % simple URL typesetting

\usepackage{mathtools, cuted}

\usepackage{booktabs}
\usepackage{threeparttable}
\usepackage{mathtools}

\usepackage{cuted}
\usepackage{lipsum, color}
\usepackage{amsthm}

\usepackage{amsfonts,bm}

\def\eqref#1{equation~\ref{#1}}
\def\1{\bm{1}}

\usepackage{algorithm}
\usepackage[noend]{algpseudocode}

\def\vc{{\bm{c}}}
\def\vd{{\bm{d}}}

\def\vg{{\bm{g}}}

\def\vn{{\bm{n}}}

\def\vu{{\bm{u}}}
\def\vv{{\bm{v}}}
\def\vw{{\bm{w}}}
\def\vx{{\bm{x}}}
\def\vy{{\bm{y}}}

\def\mA{{\bm{A}}}
\def\mB{{\bm{B}}}

\def\mD{{\bm{D}}}

\def\mH{{\bm{H}}}
\def\mI{{\bm{I}}}
\def\mJ{{\bm{J}}}

\def\mL{{\bm{L}}}
\def\mM{{\bm{M}}}

\def\mT{{\bm{T}}}

\def\mX{{\bm{X}}}

\DeclareMathAlphabet{\mathsfit}{\encodingdefault}{\sfdefault}{m}{sl}
\SetMathAlphabet{\mathsfit}{bold}{\encodingdefault}{\sfdefault}{bx}{n}

\newcommand{\E}{\mathbb{E}}

\newcommand{\R}{\mathbb{R}}

\newcommand{\Var}{\mathrm{Var}}

\newcommand{\w}{\vw}
\newcommand{\g}{\vg}
\let\d\relax
\newcommand{\d}{\vd}
\newcommand{\A}{\mA}
\newcommand{\B}{\mB}
\let\L\relax
\newcommand{\L}{\mL}

\newcommand{\D}{\mD_+}
\newcommand{\I}{\mI}

\newcommand{\0}{\mathbf{0}}
\let\t\relax
\newcommand{\t}{\top}

\newcommand{\lla}{\langle}
\newcommand{\rra}{\rangle}

\DeclareMathOperator{\trace}{trace}
\DeclareMathOperator{\DTFT}{\mathrm{DTFT}}

\def\R{\mathbb{R}}

\newtheorem{proposition}{Proposition}
\newtheorem{definition}{Definition}
\newtheorem{theorem}{Theorem}
\newtheorem{lemma}{Lemma}
\newtheorem{prop}{Proposition}
\newtheorem{remark}{Remark}
\newtheorem{innercustomprop}{Proposition}
\newenvironment{customprop}[1]
  {\renewcommand\theinnercustomprop{#1}\innercustomprop}
  {\endinnercustomprop}

\usepackage{tikz}
\usetikzlibrary{shapes.geometric, arrows}
\tikzstyle{normal} = [rectangle, rounded corners, minimum width=1.0cm, minimum height=0.8cm,text centered, draw=black]
\tikzstyle{arrow} = [thick,->,>=stealth]

\title{Laplacian Smooth Gradient Descent}

\author{
Stanley J. Osher \\
  Department of Mathematics\\
  University of California, Los Angeles\\
  \texttt{sjo@math.ucla.edu}\\
  \and
Bao Wang \\
  Department of Mathematics\\
  University of California, Los Angeles\\
  \texttt{wangbaonj@gmail.com}\\
  \and
  Penhang Yin \\
  Department of Mathematics\\
  University of California, Los Angeles\\
  \texttt{yph@ucla.edu}\\
  \and
  Xiyang Luo \\
  Department of Mathematics\\
  University of California, Los Angeles\\
  \texttt{xylmath@gmail.com}\\
  \and
  Farzin Barekat \\
  Department of Mathematics\\
  University of California, Los Angeles\\
  \texttt{fbarekat@math.ucla.edu}\\
  \and
  Minh Pham \\
  Department of Mathematics\\
  University of California, Los Angeles\\
  \and
  Alex Lin \\
  Department of Mathematics\\
  University of California, Los Angeles\\
}

\begin{document}

\maketitle

\begin{abstract}
We propose a class of very simple modifications of gradient descent and stochastic gradient descent. We show that when applied to a large variety of machine learning problems, ranging from logistic regression to deep neural nets, the proposed surrogates can dramatically reduce the variance, allow to take a larger step size, and improve the generalization accuracy. The methods only involve multiplying the usual (stochastic) gradient by the inverse of a positive definitive matrix (which can be computed efficiently by FFT) with a low condition number coming from a one-dimensional discrete Laplacian or its high order generalizations. It also preserves the mean and increases the smallest component and decreases the largest component. The theory of Hamilton-Jacobi partial differential equations demonstrates that the implicit version of the new algorithm is almost the same as doing gradient descent on a new function which (i) has the same global minima as the original function and (ii) is ``more convex". Moreover, we show that optimization algorithms with these surrogates converge uniformly in the discrete Sobolev $H_\sigma^p$ sense and reduce the optimality gap for convex optimization problems. The code is available at: \url{https://github.com/BaoWangMath/LaplacianSmoothing-GradientDescent}
\end{abstract}

\section{Introduction}\label{section-introduction}
Stochastic gradient descent (SGD) \cite{Robinds:1951}
has been the workhorse for solving large-scale machine learning (ML) problems. It gives rise to a family of algorithms that enables efficient training of many ML models including deep neural nets (DNNs). SGD utilizes training data very efficiently at the beginning of the training phase, as it converges much faster than GD and L-BFGS during this period \cite{Bottou:2018SiamReview,Hardt:2016}. Moreover, the variance of SGD can help gradient-based optimization algorithms circumvent local minima and saddle points and reach those that generalize well \cite{Schmidhuber:2014,Stanislaw:2018}. However, the variance of SGD also slows down the convergence after the first few training epochs. To account for the effect of SGD's variance and to ensure the convergence of SGD, a decaying step size has to be applied which is one of the major bottlenecks for the fast convergence of SGD \cite{Bottou:2012,Shapiro:1996,Shamir:2013}. Moreover, in training many ML models, typically the stage-wise schedule of learning rate is used in practice \cite{Senior:2013,Schmidhuber:2014}. In this scenario, the variance of SGD usually leads to a large optimality gap.
\medskip

A natural question arises from the above bottlenecks of SGD is: {\bf Can we improve SGD such that the variance of the stochastic gradient is reduced on-the-fly with negligible extra computational and memory overhead and a larger step size is allowed to train ML models?}
\medskip

We answer the above question affirmatively by applying the discrete one-dimensional Laplacian smoothing (LS) operator to smooth the stochastic gradient vector on-the-fly. The LS operation can be performed efficiently by using the fast Fourier transform (FFT). It is shown that the LS %is provably to 
reduces the variance of stochastic gradient and allows to take a larger step size.
\medskip

Another issue of standard GD and SGD is that when the Hessian of the objective function has a large condition number, gradient descent performs poorly. In this case, the derivative increases rapidly in one direction, while growing slowly in another. As a by-product, numerically we will show that LS can avoid oscillation along steep directions and help make progress in shallow directions effectively \cite{StanfordCS231n}. The implicit version of our proposed approach is linked to an unusual Hamilton-Jacobi partial differential equation (HJ-PDE) whose solution makes the original loss function more convex while retaining its flat (and global) minima, and essentially works on this surrogate function with a much better landscape. See \cite{Chaudhari:2017DeepRelaxation} for earlier, related work. %Therefore, the implicit version of LS(S)GD can help to avoid local sharp (spurious) minima. 

\subsection{Our contribution}
In this paper, we propose a new modification to the stochastic gradient-based algorithms, which at its core uses the LS operator to reduce the variance of stochastic gradient vector on-the-fly. The (stochastic) gradient smoothing can be done by multiplying the gradient by the inverse of the following circulant convolution matrix
\begin{equation}\label{eq:tri-diag}
\A_\sigma := 
\begin{bmatrix}
1+2\sigma   & -\sigma &  0&\dots &0& -\sigma \\
-\sigma     & 1+2\sigma & -\sigma & \dots &0&0 \\
0 & -\sigma  & 1+2\sigma & \dots & 0 & 0 \\
\dots     & \dots & \dots &\dots & \dots & \dots\\
-\sigma     &0& 0 & \dots &-\sigma & 1+2\sigma
\end{bmatrix}
\end{equation}
for some positive constant $\sigma \geq 0$. In fact, we can write
$\A_\sigma = \I-\sigma \L$,
where $\I$ is the identity matrix, and $\L$ is the discrete one-dimensional Laplacian which acts on indices. If we define the (periodic) forward finite difference matrix as
\begin{equation*}
\D = \begin{bmatrix}
-1   & 1 &  0&\dots &0& 0 \\
0     & -1 & 1 & \dots &0&0 \\
0 & 0  & -1 & \dots & 0 & 0 \\
\dots     & \dots & \dots &\dots & \dots & \dots\\
1     &0& 0 & \dots &0 & -1
\end{bmatrix}.
\end{equation*}
Then, we have $\A_\sigma = \I - \sigma \mD_-\D$, where $\mD_- = -\D^\t$ is the backward finite difference.
\medskip

We summarize the benefits of this simple LS operation below:
\begin{itemize}
\item It reduces the variance of stochastic gradient on-the-fly, and reduces the optimality gap when constant step size is used.
\item It allows us to take a larger step size than the standard (S)GD.
\item It is applicable to train a large variety of ML models including DNNs with better generalization.
\item It converges faster for the objective functions that have a large condition number numerically.
\item It avoids local sharp minima empirically.
\end{itemize}
\medskip

Moreover, as a straightforward extension, we generalize the LS to high-order smoothing operators, e.g., biharmonic smoothing.

\subsection{Related work}
There is an extensive volume of research over the past decades for designing algorithms to speed up the convergence. These include using momentum and other heavy-ball methods, reduce the variance of the stochastic gradient, and adaptive the learning rate. We will discuss the related work from these three perspectives.
\medskip

The first type of idea to accelerate the convergence of GD and SGD is to apply the momentum. Around local optima, the surface curves can be much more steeply in one dimension than in another \cite{Sutton:1986}, whence (S)GD oscillates across the slopes of the ravine while only making hesitant progress along the bottom towards the local optimum. Momentum is proposed to accelerate (S)GD in the relevant direction and dampens oscillations \cite{Qian:1999}. Nesterov accelerated gradient (NAG) is also introduced to slow down the progress before the surface curve slopes up, and it provably converge faster in specific scenarios \cite{Nesterov:1983}. There are lots of recent progress in the development of momentum; a relatively complete survey can be found at \cite{Katyusha:2018}.
\medskip

Due to the bottleneck of the variance of the stochastic gradient, a natural idea is to reduce the variance of the stochastic gradient. There are several principles in developing variance reduction algorithms, including Dynamic sample size methods; Gradient aggregation, control variate type of technique is widely used along this direction, some representative works are SAGA \cite{SAGA}, SCSG \cite{SCSG}, and SVRG \cite{SVRG}; Iterative averaging methods. A thorough survey can be found at \cite{Bottou:2018SiamReview}.
\medskip

Another category of work tries to speed up the convergence of GD and SGD by using an adaptive step size, which makes use of the historical gradient to adapt the step size. RMSProp \cite{tieleman2012lecture} and Adagrad \cite{Adagrad:2011} adapts the learning rate to the parameters, performing smaller updates (i.e., low learning rates) for parameters associated with frequently occurring features, and more substantial updates (i.e., high learning rates) for parameters associated with infrequent features. Both RMSProp and Adagrad make the learning rate to be historical gradient dependent. Adadelta \cite{Adadelta:2012} extends the idea of RMSProp and Adagrad, instead of accumulating all past squared gradients, it restricts the window of accumulated past gradients to some fixed size $w$. Adam \cite{Adam:2014} and AdaMax \cite{Adam:2014} behave like a heavy ball with friction, and they compute the decaying averages of past and past squared gradients to adaptive the learning rate. AMSGrad \cite{Reddi:2018} fix the issue of Adam that may fail to converge to an optimal solution. Adam can be viewed as a combination of RMSprop and momentum: RMSprop contributes the exponentially decaying average of past squared gradients, while momentum accounts for the exponentially decaying average of past gradients. Since NAG is superior to vanilla momentum, Dozat \cite{Dozat:2016} proposed NAdam which combines the idea Adam and NAG.

% SGD with momentum. Nesterov, Different momentum, heavy-ball, Allen-Zhu's work and corresponding references

% SDCA, SAGA, SVRG, SCSG. Semi-stochastic gradient

% Ada-grad, adam, rmsprop, adadelta, etc.

\subsection{Notations}
Throughout this paper, we use boldface upper-case letters $\A$, $\B$ to denote matrices and boldface lower-case letters $\w$, $\vu$ to denote vectors. For vectors, we use $\|\cdot\|$ to denote the $\ell_2$-norm for vectors and spectral norm for matrices, respectively. And we use $\lambda_{max}(\A)$, $\lambda_{min}(\A)$, and $\lambda_i(\A)$ to denote the largest, smallest, and the $i$-th largest eigenvalues, respectively. For a function $f: \mathbb{R}^n\rightarrow \mathbb{R}$, we use $\nabla f$ and $\nabla^2 f$ to denote its gradient and Hessian, and $f^*$ to denote a local minimum of $f$. For a positive definite matrix $\A$, we define the vector induced norm by the matrix $\A$ as $\|\w\|_\A :=\sqrt{\lla \w, \A\w \rra}$. List $\{1, 2, \cdots, n\}$ is denoted by $[n]$.

\subsection{Organization}
We organize this paper as follows: In section~\ref{section-LSGD}, we introduce the LS(S)GD algorithm and the FFT-based fast solver. In section~\ref{section-step-size}, we show that LS(S)GD allows us to take a larger step size than (S)GD based on the and $\ell_2$ estimate of the introduced discrete Laplacian operator. In section~\ref{section-variance-reduction}, we show that LS reduces the variance of SGD both empirically and theoretically. We show that LSGD can avoid some local minima and speed up convergence numerically in section~\ref{section-numerics1}. In section~\ref{section-deeplearning}, we show the benefit of LS in deep learning, including training LeNet \cite{LeCun:1998}, ResNet \cite{ResNet}, Wasserstein generative adversarial nets (WGAN) \cite{Arjovsky:2017}, and deep reinforcement learning (DRL) model. The convergence analysis for LS(S)GD is provided in section~\ref{section-convergence-analysis}. The connection to the Hamilton-Jacobi partial differential equations (HJ-PDEs) and future direction are discussed in section~\ref{section-discussion-conclusion}. Most of the technical proofs are provided in section~\ref{section-appendix}.

\section{Laplacian Smoothing (Stochastic) Gradient Descent}\label{section-LSGD}
We present our algorithm for SGD in the finite-sum setting. The GD and other settings follow straightforwardly. Consider the following finite-sum optimization
\begin{align}\label{Finite-Sum}
    \min_{\vw} F(\vw) :=\frac{1}{n}\sum_{i=1}^nf_i(\vw),
\end{align}
where $f_i(\vw) \doteq f(\vw, \vx_i, y_i)$ is the loss of a given ML model on the training data $\{\vx_i, y_i\}$. This finite-sum formalism is an
abstract of training many ML models mentioned above. To resolve the optimization problem Eq.~(\ref{Finite-Sum}), starting from some initial guess $\vw^0$, the $(k+1)$-th iteration of SGD reads
\begin{equation}
\vw^{k+1} = \vw^k - \eta_k \nabla f_{i_k}(\vw^k),
\end{equation}
where $\eta_k$ is the step size, $i_k$ is a random sample with replacement from $[n]$.
\medskip

We propose to replace the stochastic gradient $\nabla f_{i_k}(\vw^k)$ by the Laplacian smoothed surrogate, and we call the resulting algorithm LSSGD, which is written as
\begin{equation}
\vw^{k+1} = \vw^k - \eta_k \A_\sigma^{-1} \nabla f_{i_k}(\vw^k).
\end{equation}
Intuitively, compared to the standard GD, this scheme smooths the gradient on-the-fly by an elliptic smoothing operator while preserving the mean of the entries of the gradient. We adopt fast Fourier transform (FFT) to compute $\A_\sigma^{-1}\nabla f(\w^{k})$, which is available in both PyTorch \cite{paszke2017automatic} and TensorFlow \cite{Tensorflow:2016}. Given a vector $\vg$, a smoothed vector $\vd$ can be obtained by computing $\vd = \A_\sigma^{-1}\vg$. This is equivalent to $\vg=\vd -\sigma \vv*\vd$, where $\vv = [-2, 1, 0, \cdots, 0, 1]^\t$ and $*$ is the convolution operator. Therefore
$$
\vd = {\rm ifft}\left(\frac{{\rm fft}(\vg)}{\mathbf{1} -\sigma \cdot {\rm fft}(\vv)}\right),
$$
where we use component-wise division (here, ${\rm fft}$ and ${\rm ifft}$ are the FFT and inverse FFT, respectively). Hence, the gradient smoothing can be done in quasilinear time. This additional time complexity is almost the same as performing a one step update on the weights vector $\vw$. For many machine learning models, we may need to concatenate the parameters into a vector. This reshaping might lead to some ambiguity, nevertheless, based on our tests, both row and column majored reshaping work for the LS-GD algorithm. Moreover, in deep learning cases, the weights in different layers might have different physical meanings. For these cases, we perform layer-wise gradient smoothing, instead. We summarize the LSSGD for solving the finite-sum optimization Eq.~(\ref{Finite-Sum}) in Algorithm~\ref{LSSGD-Pseudocode}.

\begin{algorithm}[t]
\caption{LSSGD}\label{LSSGD-Pseudocode}
\begin{algorithmic}
\State \textbf{Input: } $f_i(\vw)$ for $i=1, 2, \cdots, n$. \\
$\vw^0$: initial guess of $\vw$, $T$: the total number of iterations, and $\eta_k$, $k=0, 1, \cdots, T$: the scheduled step size.
\State \textbf{Output: } The optimized weights $\vw^{\rm opt}$.
\For {$k=0, 1, \cdots, T$}
\State $\vw^{k+1} = \vw^k - \eta\A_\sigma^{-1}\left(\nabla f_{i_k}(\vw^k)\right)$.
\EndFor
\Return $\vw^T$
\end{algorithmic}
\end{algorithm}

\begin{remark}
In image processing and elsewhere, the Sobolev gradient \cite{Jung:2009} uses a multi-dimensional Laplacian operator that operates on $\w$, and is different from the one-dimensional discrete Laplacian operator employed in our LS-GD scheme that operates on indices.
\end{remark}

It is worth noting that LS is a complement to the heavy ball, e.g., Nesterov momentum, and adaptive learning rate, e.g., Adam, algorithms. It can be combined with these acceleration techniques to speed up the convergence. We will show the performance of these algorithms in the Section~\ref{section-deeplearning}.

% Insert finite-sum formula here

% Insert a pseudo-code here, input T, $\eta_k$ k = 1, 2, ..., T, objective function, etc

\subsection{Generalized smoothing gradient descent}
We can generalize $\A_\sigma$ to the $n$-th order discrete hyper-diffusion operator as follows
$$
\I + (-1)^n \sigma \L^n \doteq \A_\sigma^n.
$$
Each row of the discrete Laplacian operator $\L$ consists of an appropriate arrangement of weights in central finite difference approximation to the 2nd order derivative. Similarly, each row of $\L^n$ is an arrangement of the weights in the central finite difference approximation to the $2n$-th order derivative.

\begin{remark}
The $n$-th order smoothing operator $\I + (-1)^n \sigma \L^n$ can only be applied to the problem with dimension at least $2n+1$. Otherwise, we need to add dummy variables to the object function.
\end{remark}

Again, we apply FFT to compute the smoothed gradient vector. 
For a given gradient vector $\vg$, the smoothed surrogate, $(\A_\sigma^n)^{-1}\vg\doteq \vd$, can be obtained by solving 
$\vg=\vd+(-1)^n\sigma \vv_n*\vd$, where $\vv_n = (c^n_{n+1}, c^n_{n+2}, \cdots, c^n_{2n+1}, 0, \cdots, 0, c^n_1, c^n_2, \cdots, c^n_{n-1}, c^n_n)$ is a vector of the same dimension as the gradient to be smoothed. And the coefficient vector $\vc^n = (c^n_1, c^n_2, \cdots, c^n_{2n+1})$ can be obtained recursively by the following formula
\begin{equation*}
\vc^1 = (1, -2, 1), \quad
c^n_i = 
\begin{cases}
1 &\text{ $i = 1, 2n+1$}\\
-2c^{n-1}_1+c^{n-1}_2 &\text{ $i = 2, 2n$}\\
c^{n-1}_{i-1}-2c^{n-1}_i + c^{n-1}_{i+1} &\text{otherwise.}
\end{cases}
\end{equation*}

\begin{remark}
The computational complexities for different order smoothing schemes are the same when the FFT is utilized for computing the surrogate gradient. 
\end{remark}

\section{The Choice of Step Size}\label{section-step-size}
In this section, we will discuss the step size issue of LS(S)GD with a theoretical focus on LSGD on $L$-Lipschitz functions.

\begin{definition}[$L$-Lipschitz]
We say the function $F$ is $L$-Lipschitz, if for any  $\w, \vu\in\R^m$, we have $\|f(\w) -  f(\vu)\|\leq L\|\w - \vu\|$.
\end{definition}

\begin{remark}
If the function $F$ is $L$-Lipschitz and differentiable, then for any $\w$, we have
$\|\nabla f(\w)\| \leq L$.
\end{remark}

For $L$-Lipschitz function, it is known that the largest suitable step size for GD is $\eta_{max}^{GD} = \frac{1}{L}$ \cite{Nesterov:1998}. In the following, we will establish a $\ell_2$ estimate of the square root of the LS operator when it is applied to an arbitrary vector. Based on these estimates, we will show that LSGD can take a larger step size than GD.
\medskip

To determine the largest suitable step size for LSGD. We first do a change of variable in the LSGD \ref{Finite-Sum} by letting $\vv^k = \mH^{-1/2}_\sigma\w^k$ where $\mH_\sigma=\A_\sigma^{-1}$, then LSGD can be written as
\begin{equation}
\label{LSGD-vk}
\vv^{k+1} = \vv^k - \eta_k \mH^{1/2}_\sigma\nabla F(\mH^{1/2}_\sigma\vv^k),
\end{equation}
which is actually the GD for solving the following minimization problem
\begin{equation}
\label{GD-Hv}
\min_\vv F(\mH^{1/2}_\sigma\vv) := \min_{\vv} G(\vv).
\end{equation}
Therefore, to determine the largest suitable step size for LSGD, it is equivalent to find the largest appropriate step size for GD for $\min_{\vv} G(\vv)$. Therefore, it suffices to determine the Lipschitz constant for the function $G(\vv)$, i.e., to find
$$
L_G := \inf_\vv\left\{\|\nabla G(\vv)\| | \vv \in {\rm dom}(G) \right\}.
$$
Note that for $\forall \vv_1, \vv_2$, we have
\begin{eqnarray*}
\|G(\vv_1)-G(\vv_2)\| &=&    \|F(\mH^{1/2}_\sigma\vv_1) - F(\mH^{1/2}_\sigma\vv_2)\|\\
                      &\leq& L\|\mH^{1/2}_\sigma\vv_1-\mH^{1/2}_\sigma\vv_2\|
\end{eqnarray*}
To find the largest appropriate step size, we need to further estimate $\|\mH^{1/2}_\sigma\vv_1-\mH^{1/2}_\sigma\vv_2\|$.

\subsection{$\ell_2$ estimates of $\mathbf{H}_\sigma^{1/2}\mathbf{v}$}

\begin{proposition}\label{Prop-L2-Estimate}
Given any vector $\vv\in \mathbb{R}^m$, let $\vw = \A_\sigma^{-1/2}\vv$, then
\begin{equation}
\label{L2-Estimate-1}
\|\vv\|^2 = \|\vw\|^2 + \sigma \|\mD_+\vw\|^2.
\end{equation}
\end{proposition}

\begin{proof}
Observe that $\vv = A_\sigma^{1/2} \vw$. Therefore,
\begin{align*}\label{eqn:gsquared_onehalf}
\|\vv\|^2 =& \left \langle \A_\sigma^{1/2} \vw , \A_\sigma^{1/2} \vw  \right\rangle
=  \left \langle \A_\sigma \vw , \vw  \right\rangle
=  \langle \vw - \sigma \mD_{-}\mD_{+} \vw, \vw \rangle
=   \|\vw\|^2 - \sigma \langle \mD_{-}\mD_{+}\vw,\vw\rangle \\
= &   \|\vw\|^2 - \sigma \langle \mD_{+}\vw,-\mD_{+}\vw\rangle
=   \|\vw\|^2 + \sigma \|\mD_{+}\vw\|^2,
\end{align*}
where we used $\mD_{-}^T=-\mD_{+}$ for the second last equality.
\end{proof}

Proposition~\ref{Prop-L2-Estimate} shows that the Lipschitz constant of $G$ is not larger than that of $F$, since
$$
\|\mH^{1/2}_\sigma\vv_1-\mH^{1/2}_\sigma\vv_2\|^2 = \|\vv_1-\vv_2\|^2 - \sigma \|\mD_+( \mH^{1/2}_\sigma\vv_1-\mH^{1/2}_\sigma\vv_2 )\|^2 \leq \|\vv_1-\vv_2\|^2.
$$
Therefore, LSGD can take at least the same step size as GD. However, note that $\|\mD_+\vw\|_2$ can be arbitrarily close to zero, so LSGD cannot always take a larger step size than GD. Next, we establish a high probability estimation for taking a larger step size when using LSGD.
\medskip

Without any prior knowledge about $\vv_1-\vv_2:=\vv$, let us assume it is sampled uniformly from a ball in $\mathbb{R}^m$ centered at the origin. Without loss of generality, we assume the radius of this ball is one. For the sake of notation simplicity, in the following we denote $\mH^{1/2}_\sigma := \mM_\sigma$. Under the above ansatz, we have the following result

\begin{theorem}[$\ell_2$-estimate]\label{L2-Results-High-Prob}
Let %$0< \epsilon <1-\frac{\pi}{\sqrt{m}}$, 
$\sigma>0$, and 
$$
\beta = \frac{1}{m}\sum_{i=1}^m \frac{1}{1+2\sigma-\sigma z_i - \sigma \overline{z_i}},
$$
where $z_1$, $\cdots$, $z_m$ are the $m$ roots of unity. Let $\vv$ be uniformly distributed in the unit ball of the $m$ dimensional $\ell_2$ space. Then
\begin{equation}
\label{High-Prop-Estimate-L2}
\mathbb{P}\left( \|\mM_\sigma \vv\| \geq \alpha \|\vv\| \right) \leq 2\exp{\left( -\frac{2}{\pi^2}m\left( \frac{\alpha-\alpha\frac{\pi}{\sqrt{m}} -\sqrt{\beta} }{\alpha+1} \right)^2 \right)}
\end{equation}
for any $\alpha > \frac{\sqrt{\beta}}{1-\frac{\pi}{\sqrt{m}}}$.
\end{theorem}

The proof of this theorem is provided in the appendix. For high dimensional ML problems, e.g., training DNNs, $m$ can be as large as tens of millions so that the probability will be almost one. The closed form of $\beta$ is given in Lemma~\ref{thm:closed_formula}.

\begin{lemma}\label{thm:closed_formula}
If $z_1,\ldots,z_m$ denote the $m$ roots of unity, then
\begin{equation}\label{eqn:closed_formula}
\beta = \frac{1}{m}\sum_{j=1}^{m} \frac{1}{1+2\sigma-\sigma z_j - \sigma \bar{z_j}} = \frac{1+\alpha^m}{(1-\alpha^m)\sqrt{4\sigma+1}} \rightarrow \frac{1}{\sqrt{1+4\sigma}},
\end{equation}
as $m\rightarrow \infty $, where 
$$
1> \alpha = \frac{2\sigma+1 - \sqrt{4\sigma+1}}{2\sigma} > 0.
$$
\end{lemma}

The proof of the above lemma requires some tools from complex analysis and harmonic analysis, which is provided in the appendix. Table~\ref{Beta-Table} lists some typical values for different $\sigma$ and dimensions $m$.
\medskip 

\begin{table}[!h]
\centering
\fontsize{10}{10}\selectfont
\begin{threeparttable}
\caption{The values of $\beta$ corresponding to some $\sigma$ and $m$. $\beta$ converges quickly to its limiting value as $m$ increases.}\label{Beta-Table}
\begin{tabular}{cccccc}
\toprule[1.0pt]
$\sigma$   &\ \ \ \ \ \ \ \  1 \ \ \ \ \ \ \ \ &\ \ \ \ \ \ \ \  2\ \ \ \ \ \ \ \  &\ \ \ \ \ \ \ \  3\ \ \ \ \ \ \ \  &\ \ \ \ \ \ \ \  4\ \ \ \ \ \ \ \  &\ \ \ \ \ \ \ \  5\ \ \ \ \ \ \ \   \cr
\midrule[0.8pt]
$m=1000$    & 0.447 & 0.333 & 0.277 & 0.243 & 0.218 \cr
$m=10000$   & 0.447 & 0.333 & 0.277 & 0.243 & 0.218 \cr
$m=100000$  & 0.447 & 0.333 & 0.277 & 0.243 & 0.218 \cr
\bottomrule[1.0pt]
\end{tabular}
\end{threeparttable}
\end{table}

Based on the estimate in Theorem~\ref{L2-Results-High-Prob}, LSGD can take the largest step size $\frac{1}{\sqrt{\beta}L}$ for high-dimensional $L$-Lipschitz function with high probability.
We will verify this result numerically in the following sections.

\section{Variance Reduction}\label{section-variance-reduction}
The variance of SGD is one of the major bottlenecks that slows down the theoretical guaranteed convergence rate in training ML models. Most of the existing variance reduction algorithms require either the full batch gradient or the storage of stochastic gradient for each data point which makes it difficult to be used to train the high-capacity DNNs. LS is an alternative approach to reduce the variance of the stochastic gradient with negligible extra computational time and memory cost. In this section, we rigorously show that LS reduces the variance of the stochastic gradient and reduce the optimality gap under the Gaussian noise assumption. Moreover, we numerically verify our theoretical results on both a quadratic function and a simple finite-sum optimization problem.

\subsection{Gaussian noise assumption}
Stochastic gradient $\nabla f_{i_k}$, for any $i_k \in [n]$, is an unbiased estimate of $\nabla F$, many existing works model the variance between the stochastic gradient and full batch gradient $\nabla F$ as Gaussian noise $\mathcal{N}(\mathbf{0}, \Sigma)$, where $\Sigma$ is the covariance matrix \cite{Mandt:2018}. Therefore, ignoring the variable $\vw$ for simplicity of notation,
%without ambiguity we ignore the variable $\vw$, 
we can write the equation involving gradient and stochastic gradient vectors as
\begin{equation}
\label{GD-SGD}
\nabla f_{i_k} = \nabla F + \vn,
\end{equation}
where $\vn \sim \mathcal{N}(\mathbf{0}, \Sigma)$. 
%Similar for LS gradient and LS stochastic gradient vectors, we have
Thus for LS stochastic gradient, we have
\begin{equation}
\label{LSGD-LSSGD}
\A_\sigma^{-1}\nabla f_{i_k} = \A_\sigma^{-1}\left(\nabla F + \vn\right).
\end{equation}
The variances of stochastic gradient and LS stochastic gradient are basically the variance of $\vn$ and $\A_\sigma^{-1}\vn$, respectively. The following theorem quantifies the variance between $\vn$ and $\A_\sigma^{-1}\vn$.

\begin{theorem}\label{Theorem-Variance-Reduction}
Let $\kappa$ denote the condition number of $\Sigma$. Then, for $m$ dimensional Gaussian random vector $\vn\sim \mathcal{N}(\mathbf{0}, \Sigma)$, we have
\begin{equation}\label{eqn:var_reduction_formula2}
\frac{\sum_{i=1}^m \Var[\left((\A_\sigma^n)^{-1} \vn\right)_i]}{\sum_{i=1}^m\Var[ \left(\vn\right)_i ]} \leq 1 - \frac{1}{\kappa} +  \frac{1}{\kappa m}\sum_{j=0}^{m}\frac{1}{[1+4^n\sigma \sin^{2n}(\pi j/m)]^2}.
\end{equation}
\end{theorem}

The proof of Theorem~\ref{Theorem-Variance-Reduction} will be provided in the appendix.
\medskip 

Table~\ref{Variance-Reduction-Table} lists the ratio of variance after and before LS for an $m$-dimensional standard normal vector, i.e., $\vn\sim \mathcal{N}(\mathbf{0}, \mI)$. In practice, high order smoothing reduce variance more significantly.

\begin{table}[!h]
\centering
\fontsize{10}{10}\selectfont
\begin{threeparttable}
\caption{Theoretical upper bound of $\sum_{i=1}^m \Var[\left((\A_\sigma^n)^{-1} \vn\right)_i]/\sum_{i=1}^m\Var[ \left(\vn\right)_i ]$ when $\mathbf{n}$ is an $m$-dimensional standard normal vector with $m\geq 10000$.}\label{Variance-Reduction-Table}
\begin{tabular}{cccccc}
\toprule[1.0pt]
$\sigma$   &\ \ \ \ \ \ \ \  1 \ \ \ \ \ \ \ \ &\ \ \ \ \ \ \ \  2\ \ \ \ \ \ \ \  &\ \ \ \ \ \ \ \  3\ \ \ \ \ \ \ \  &\ \ \ \ \ \ \ \  4\ \ \ \ \ \ \ \  &\ \ \ \ \ \ \ \  5\ \ \ \ \ \ \ \   \cr
\midrule[0.8pt]
$n=1$  & 0.268 & 0.185 & 0.149 & 0.129 & 0.114 \cr
$n=2$  & 0.279 & 0.231 & 0.207 & 0.192 & 0.181 \cr
$n=3$  & 0.290 & 0.256 & 0.238 & 0.226 & 0.218 \cr
\bottomrule[1.0pt]
\end{tabular}
\end{threeparttable}
\end{table}

\medskip

Moreover, LS preserves the mean (Proposition~\ref{prop:max}), decreases the largest component and increases the smallest component (Proposition~\ref{prop:conser}) for any vector.

\begin{proposition}\label{prop:max}
For any vector $\g\in\R^m$, $\d = \A_\sigma^{-1}\g$, 
let $j_{\max} = \arg\max_i d_i$ and $j_{\min} = \arg\min_i d_i $.
We have $\max_i d_i = d_{j_{\max}} \leq g_{j_{\max}} \leq \max_i g_i$ and  $\min_i d_i = d_{j_{\min}} \geq g_{j_{\min}} \geq \min_i g_i$.
\end{proposition}

\begin{proof}%[\bf Proof of Proposition \ref{prop:max}]
Since $\g = \A_{\sigma}\d$, it holds that 
$$
g_{j_{\max}} = d_{j_{\max}} + \sigma( 2 d_{j_{\max}} - d_{j_{\max}-1} - d_{j_{\max}+1}),
$$
where periodicity of subindex are used if necessary. Since $ 2 d_{j_{\max}} - d_{j_{\max}-1} - d_{j_{\max}+1}\geq 0$, We have $\max_i d_i = d_{j_{\max}} \leq g_{j_{\max}} \leq \max_i g_i$. A similar argument can show that $\min_i d_i = d_{j_{\min}} \geq g_{j_{\min}} \geq \min_i g_i$.
\end{proof}

\begin{proposition}\label{prop:conser}
The operator $\A_\sigma^{-1}$ preserves the sum of components. For any $\g\in\R^m$ and $\d = \A_\sigma^{-1}\g$, we have $\sum_j d_j = \sum_j g_j$, or equivalently, $\1^\t \d = \1^\t \g$. 
\end{proposition}

\begin{proof}%[\bf Proof of Proposition \ref{prop:conser}]
Since $\g = \A_\sigma \d$, 
$$
\sum_i g_i = \mathbf{1}^\top \mathbf{g} = \mathbf{1}^\top (\I + \sigma \D^\t \D )\d = \mathbf{1}^\top \d = \sum_i d_i,
$$
where we used $\D\mathbf{1}  = \0$.
\end{proof}

\subsection{Reduce the optimality gap}
A direct benefit of variance reduction is that it reduces the optimality gap in SGD when constant step size is applied. We state the corresponding result in the following.

\begin{proposition}\label{Prop-Optimality-Gap}
Suppose $f$ is convex with the global minimizer $\w^*$, and $f^* = f(\w^*)$. Consider the following iteration with constant learning rate $\eta>0$
$$
\w^{k+1} = \w^k - \eta (\A_\sigma^n)^{-1} \g^k
$$ 
where $\g^k$ is the sampled gradient in the $k$-th iteration at $\w^k$ satisfying $\E[\g^k] = \nabla f(\w^k)$. Denote $G_{\A_\sigma^n}: = \lim_{K\to\infty}\frac{1}{K}\sum_{k=0}^{K-1}\|\g^k\|^2_{(\A_\sigma^n)^{-1}}$ and $\overline{\w}^K := \sum_{k=0}^{K-1} \w^k /K$ the ergodic average of iterates. Then the optimality gap is
$$
\lim_{K\to \infty} \E[f(\overline{\w}^K)] - f^* \leq  \frac{\eta G_{\A_\sigma^n}}{2}.
$$
\end{proposition}

\begin{proof}
Since $f$ is convex, we have
\begin{equation}\label{eq:3}
    \langle \nabla f(\w^k), \w^{k} - \w^* \rangle\geq f(\w^k) - f^*.
\end{equation}
Furthermore, 
% by the descent lemma,
% \begin{align*}
% f(\w^{t+1})\leq &\; f(\w^t) + \langle \nabla f(\w^t), \w^{t+1} -\w^t \rangle + \frac{L}{2}\|\w^{t+1}-\w^t\|^2 \\
% \leq & \; f(\w^t)
% \end{align*}
\begin{align*}
& \; \E[\|\w^{k+1} -\w^*\|_{\A_\sigma^n}^2] = \E[\|\w^{k} -  \eta (\A_\sigma^n)^{-1}\g^k -\w^* \|_{\A_\sigma^n}^2] \\
 = & \; \E[\|\w^{k} -\w^* \|^2_{\A_\sigma^n}] - 2\eta \E[\langle \g^k, \w^k - \w^* \rangle]+ \eta^2\E[\|(\A_\sigma^n)^{-1}\g^t\|_{\A_\sigma^n}^2] \\
\leq & \; \E[\|\w^{k} -\w^* \|_{\A_\sigma^n}^2] - 2\eta \E[\langle \nabla f(\w^k), \w^k - \w^* \rangle]+ \eta^2 \|\g^k\|_{(\A_\sigma^n)^{-1}}^2 \\
\leq & \; \E[\|\w^{k} -\w^* \|_{\A_\sigma^n}^2] - 2\eta(\E[f(\w^k)]-f^*) + \eta^2 \|\g^k\|_{(\A_\sigma^n)^{-1}}^2, 
\end{align*}
where the last inequality is due to (\ref{eq:3}). We rearrange the terms and arrive at
\begin{align*}
\E[ f(\w^k)]-f^*\leq & \, \frac{1}{2\eta}(\E[\|\w^{k} -\w^* \|_{\A_\sigma^n}^2]  - \E[\|\w^{k+1} -\w^*\|_{\A_\sigma^n}^2])  + \frac{\eta  \|\g^k\|_{(\A_\sigma^n)^{-1}}^2 }{2}. 
\end{align*}
Summing over $k$ from $0$ to $K-1$ and averaging and using the convexity of $f$, we have
\begin{align*}
 \E[f(\overline{\w}^K)] - f^*\leq & \; \frac{\sum_{k=0}^{K-1} \E[f(\w^k)]}{K} - f^* \leq \frac{1}{2\eta K} \E[\|\w^{0} -\w^* \|_{\A_\sigma^n}^2] + \frac{\sum_{k=0}^{K-1}\|\g^k\|_{(\A_\sigma^n)^{-1}}^2 }{2K} \eta .
\end{align*}
Taking the limit as $K\to \infty$ above establishes the result.
\end{proof}

\begin{remark}
Since $G_{\A_\sigma^n}$ is smaller than the corresponding value without LS. It shows that the optimality gap is reduced when LS is used with a constant step size. In practice, this is also true for the stage-wise step size since it is a constant in each stage of the training phase.
\end{remark}

\subsubsection{Optimization for quadratic function}
In this part, we empirically show the advantages of the LS(S)GD and its generalized schemes for the convex optimization problems. Consider searching the minima $\vx^*$ of the quadratic function $f(\vx)$ defined in Eq.~(\ref{Convex-Test-1}).
\begin{equation}
\label{Convex-Test-1}
f(x_1, x_2, \cdots, x_{100}) = \sum_{i=1}^{50}x_{2i-1}^2 + \sum_{i=1}^{50} \frac{x_{2i}^2}{10^2}.
\end{equation}

To simulate SGD, we add Gaussian noise to the gradient vector, i.e., at any given point $\vx$, we have
$$
\tilde{\nabla}_\epsilon f(\vx) := \nabla f(\vx) + \epsilon \mathcal{N}(\mathbf{0}, \I),
$$
where the scalar $\epsilon$ controls the noise level, $\mathcal{N}(\mathbf{0}, \I)$ is the Gaussian noise vector with zero mean and unit variance in each coordinate. The corresponding numerical schemes can be formulated as
\begin{equation}
\label{Numerical-Scheme}
\vx^{k+1} = \vx^k - \eta_k (\A_\sigma^n)^{-1} \tilde{\nabla}_\epsilon f(\vx^k),
\end{equation}
where $\sigma$ is the smoothing parameter selected to be $10.0$ to remove the intense noise. We take diminishing step sizes with initial values $0.1$ for SGD/smoothed SGD; $0.9$ and $1.8$ for GD/smoothed GD, respectively. Without noise, the smoothing allows us to take larger step sizes, rounding to the first digit, $0.9$ and $1.8$ are the largest suitable step size for GD and smoothed version here. We study both constant learning rate and exponentially decaying learning rate, i.e., after every 1000 iteration the learning rate is divided by 10. We apply different schemes that corresponding to $n=0, 1, 2$ in Eq.~(\ref{Numerical-Scheme}) to the problem (Eq.~(\ref{Convex-Test-1})), with the initial point $\vx^0=(1, 1, \cdots, 1)$.
\medskip 

Figure.~\ref{fig:Optimality-Gap1} shows the iteration v.s. optimality gap when the constant learning rate is used. In the noise free case, all three schemes converge linearly.
%, but gradient smoothing has a smaller decay constant due to its increased condition number. 
When there is noise, our smoothed gradient helps to reduce the optimality gap and converges faster after a few iterations.

\begin{figure}[ht]
\centering
%\hskip -0.65cm
\begin{tabular}{cc}
\includegraphics[scale=0.45]{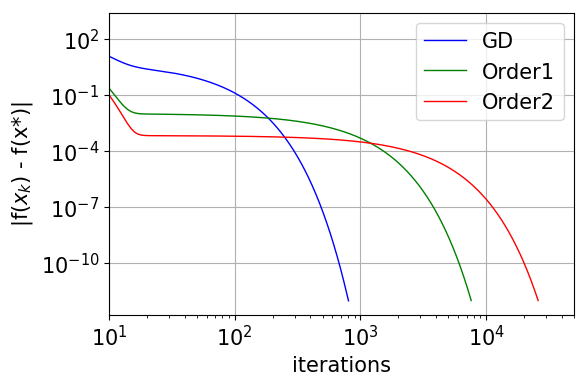} &
\includegraphics[scale=0.45]{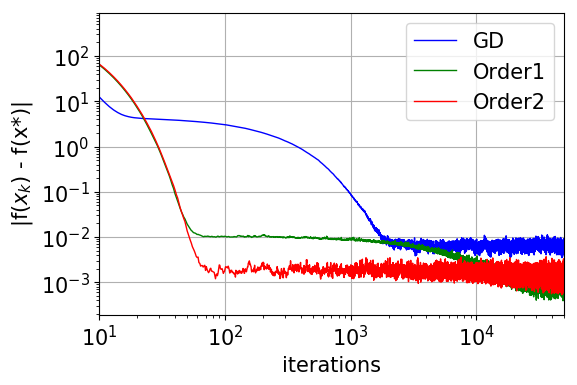} \\
(a) $\epsilon = 0$ &(b) $\epsilon = 0.05$\\
\includegraphics[scale=0.45]{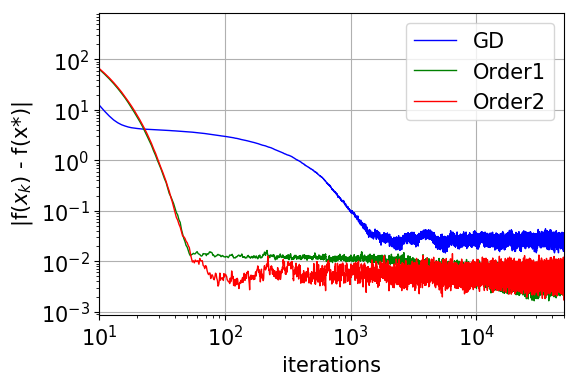} &
\includegraphics[scale=0.45]{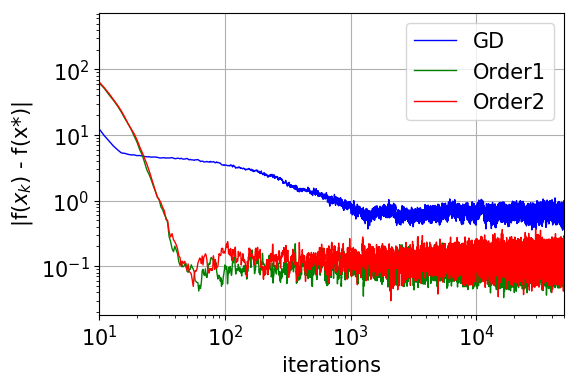}  \\
(c) $\epsilon = 0.1$ &(d) $\epsilon = 0.5$\\
\end{tabular}
%\vskip -0.32cm
\caption{Iterations v.s. optimality gap for GD and smoothed GD with order 1 and order 2 smoothing for the problem in Eq.(\ref{Convex-Test-1}). Constant step size is used.
}
\label{fig:Optimality-Gap1}
\end{figure}

The exponentially decaying learning rate helps our smoothed SGD to reach a point with a smaller optimality gap, and the higher order smoothing further reduces the optimality gap, as shown in Fig.~\ref{fig:Optimality-Gap3}. This is due to the noise removal properties of the smoothing operators.
%One simple reason for this in the noisy case is because of the noise removal properties of the smoothing operators. 

%The influence of the learning rate is still under investigation. 

\begin{figure}[ht]
\centering
%\hskip -0.65cm
\begin{tabular}{cc}
\includegraphics[scale=0.45]{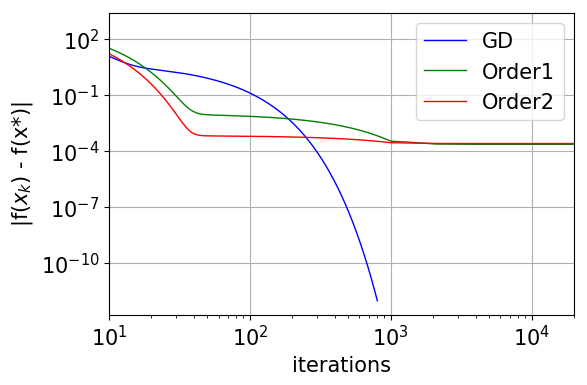} &
\includegraphics[scale=0.45]{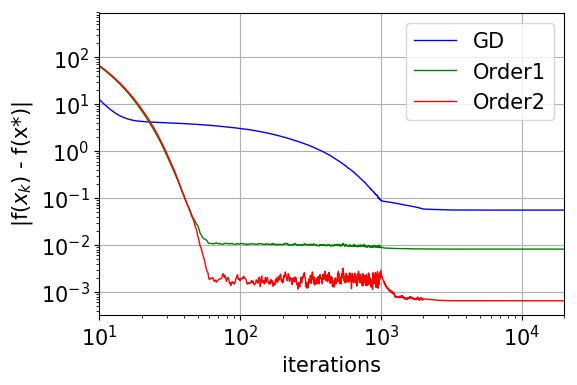} \\
(a) $\epsilon = 0$ &(b) $\epsilon = 0.05$\\
\includegraphics[scale=0.45]{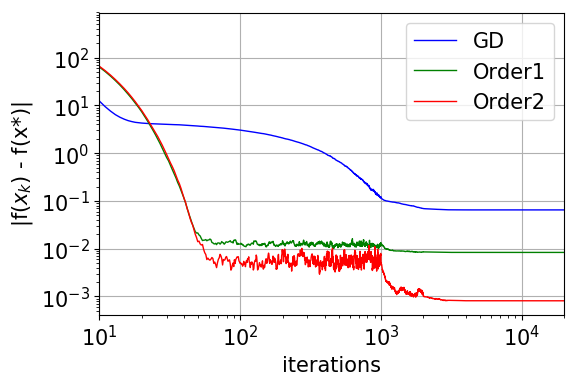} &
\includegraphics[scale=0.45]{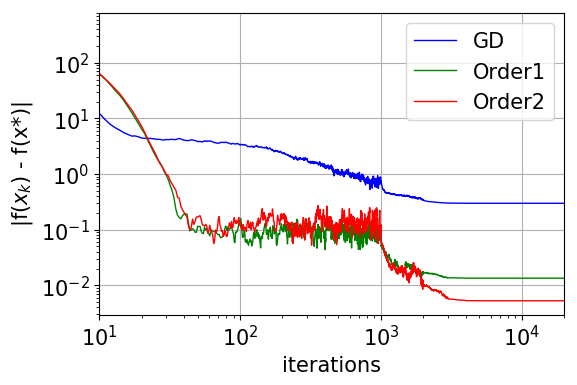}  \\
(c) $\epsilon = 0.1$ &(d) $\epsilon = 0.5$\\
\end{tabular}
%\vskip -0.32cm
\caption{Iterations v.s. optimality gap for GD and smoothed GD with order 1 and 2 smoothing for the problem in Eq.(\ref{Convex-Test-1}). Exponentially decaying step size is utilized here.
}
%\caption{Results of different optimization algorithms (gradient descent and smoothed gradient descent with orders 1 to 4) together with the exponentially decaying learning rate are applied to the problem in Eq.(\ref{Convex-Test-1}). Panels (a)-(d) plot the number of iterations v.s. optimality gap of the different schemes by using the noisy gradient.}
%\caption{Exponentially decay learning rate.}
\label{fig:Optimality-Gap3}
\end{figure}

\subsubsection{Find the center of multiple points}
Consider searching the center of a given set of 5K random points $\{\mathbf{x}_i\in \mathbb{R}^{50}\}_{i=1}^{5000}$. \footnote{We thank professor Adam Oberman for suggesting this problem to us.} This problem can be formulate as the following finite-sum optimization
\begin{equation}
\label{Find-Center}
\min_\vx F(\vx) := \frac{1}{N}\sum_{i=1}^N f_i(\vx) = \frac{1}{N}\sum_{i=1}^N \|\vx_i - \vx\|^2.
\end{equation}
We solve this optimization problem by running either SGD or LSSGD for 20K iterations starting from the same random initial point with batch size 20. The initial step size is set to be 1.0 and 1.2, respectively, for SGD and LSSGD,  and decays 1.1 times after every 10 iterations. 
As the learning rate decays, the variance of the stochastic gradient decays \cite{Welling:2011}, thus we decay $\sigma$ 10 times after every 1K iterations. Figure~\ref{fig:SGD} (a) plots a 2D cross section of the trajectories of SGD and LSSGD, and it shows that the trajectory of SGD is more noisy than that of LSSGD. Figure~\ref{fig:SGD} (b) plots the iteration v.s. loss for both SGD and LSSGD averaged over 3 independent runs. LSSGD converges faster than SGD and has a smaller optimality gap than LSSGD. This numerical result verifies our theoretical results on the optimality gap (Proposition~\ref{Prop-Optimality-Gap}).

\begin{figure}[!h]
\centering
\begin{tabular}{cc}
\includegraphics[width=0.45\columnwidth]{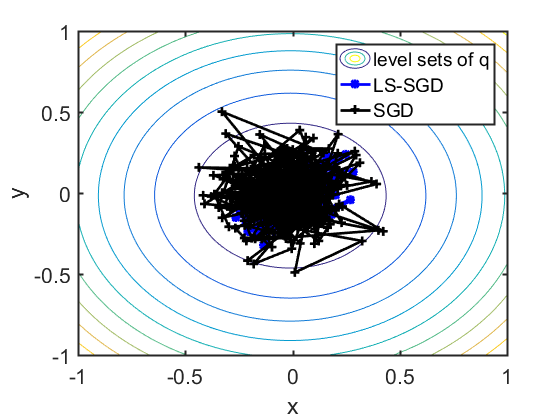}&
\includegraphics[width=0.45\columnwidth]{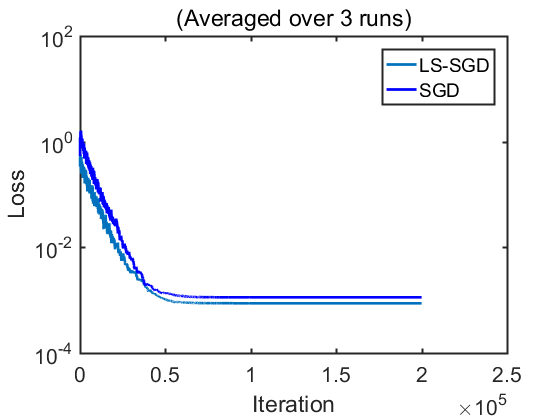}\\
\end{tabular}
\caption{Left: a cross section of the trajectories of SGD and LSSGD. Right: Iteration v.s. Loss for SGD and LS-SGD.}
\label{fig:SGD}
\end{figure}

\subsubsection{Multi-class Logistic regression}
Consider applying the proposed optimization sch--emes to train the multi-class Logistic regression model. We run 200 epochs of SGD and different order smoothing algorithms to maximize the likelihood of multi-class Logistic regression with batch size 100. 
And we apply the exponentially decaying learning rate with initial value $0.5$ and decay 10 times after every 50 epochs. We train the model with only 10 $\%$ randomly selected MNIST training data and test the trained model on the entire testing images. We further compare with SVRG under the same setting. Figure.~\ref{fig:softmax} shows the histograms of generalization accuracy of the model trained by SGD (a); SVRG (b); LS-SGD (order 1) (c); LS-SGD (oder 2) (d). It is seen that SVRG somewhat improves the generalization with higher averaged accuracy. However, the first and the second order LSSGD type algorithms lift the averaged generalization accuracy by more than $1 \%$ and reduce tnt of Electrical Engineering and Computer Sciences
University ofhe variance of the generalization accuracy over 100 independent trials remarkably. 

%The training loss of these 100 experiments by different optimization algorithms are shown in the %appendix.

\subsection{Iteration v.s. loss}
In this part, we show the evolution of the loss in training the multi-class Logistic regression model by SGD, SVRG, LSGD with first and second order smoothing, respectively. As illustrated in Fig.~\ref{fig:Iter-Loss}.
%all the optimization algorithms reduce loss of the model on the training set. 
At each iteration, among 100 independent experiments, SGD has the largest variance, SGD with first order smoothed gradient significantly reduces the variance of loss among different experiments. The second order smoothing can further reduce the variance. The variance of loss in each iteration among 100 experiments is minimized when SVRG is used to train the multi-class Logistic model. However, the generalization performance of the model trained by SVRG is not as good as the ones trained by LS-SGD, or higher order smoothed gradient descent (Fig.~\ref{fig:softmax} (b)).

\begin{figure}[ht]
\begin{center}
\begin{tabular}{cc}
\includegraphics[scale=0.45]{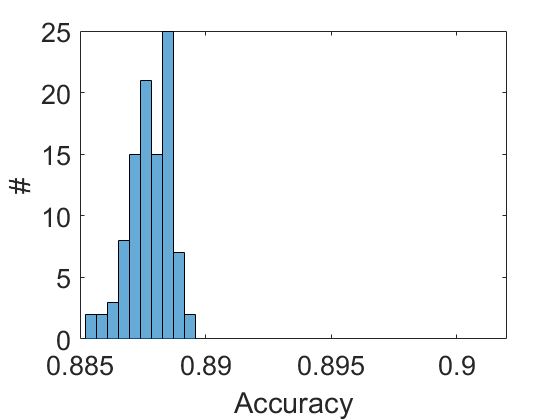} &
\includegraphics[scale=0.45]{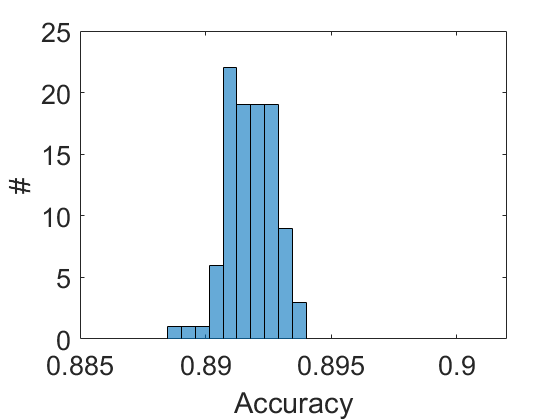} \\
(a) SGD &(b) SVRG\\
\includegraphics[scale=0.45]{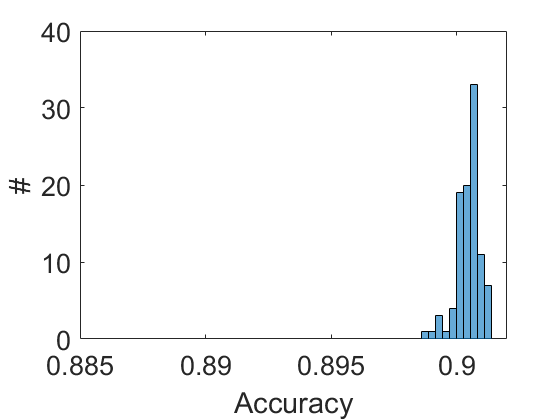} &
\includegraphics[scale=0.45]{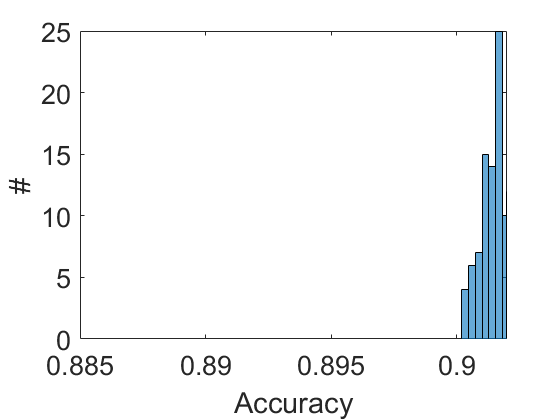}  \\
(c) LS-GD: Order 1 &(d) LS-GD: Order 2 \\
\end{tabular}
%\vskip -0.32cm
\caption{Histogram of testing accuracy over 100 independent experiments of the multi-class Logistic regression model trained on randomly selected $10\%$ MNIST data by different algorithms.}
\label{fig:softmax}
\end{center}
\end{figure}

\begin{figure}[ht]
\centering
%\hskip -0.65cm
\begin{tabular}{cc}
\includegraphics[scale=0.45]{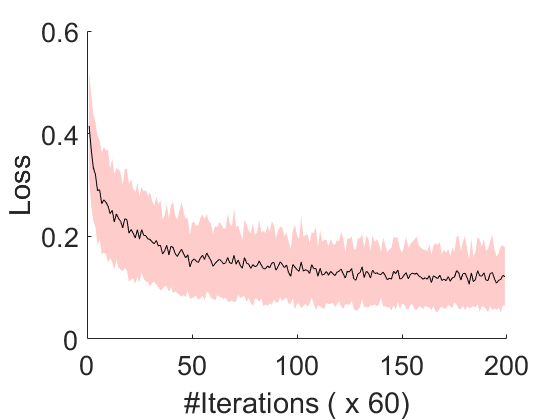} &
\includegraphics[scale=0.45]{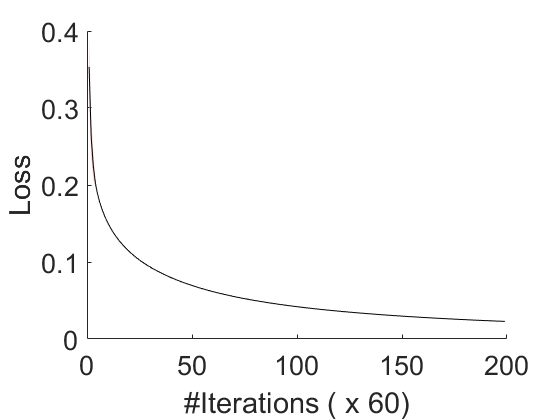} \\
(a) SGD &(b) SVRG \\
\includegraphics[scale=0.45]{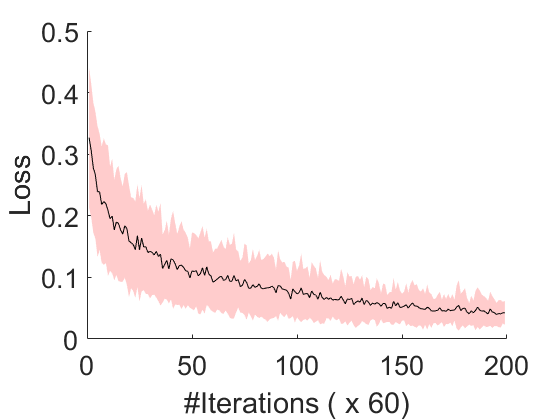} &
\includegraphics[scale=0.45]{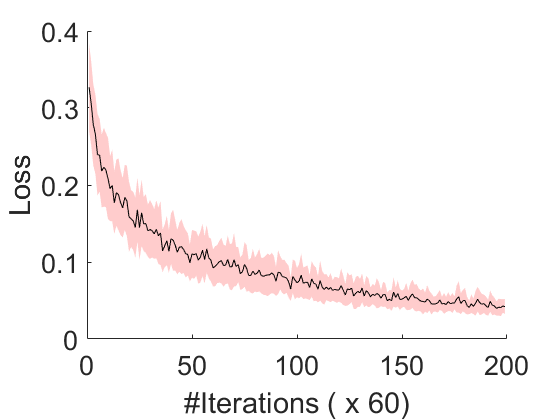}  \\
(c) LS-GD: Order 1 &(d) LS-GD: Order 2\\
\end{tabular}
\caption{Iterations v.s. loss for SGD, SVRG, and  LS-SGD with order 1 and order 2 gradient smoothing for training the multi-class Logistic regression model.
}
\label{fig:Iter-Loss}
\end{figure}

\subsection{Variance reduction in stochastic gradient}
We verify the efficiency of variance reduction numerically in this part. We simplify the problem by applying the multi-class Logistic regression only to the digits 1 and 2 of the MNIST training data. In order to compute the variance of the (LS)-stochastic gradients, we first compute descent path of (LS)-GD by applying the full batch (LS)-GD with learning rate $0.5$ starting from the same random initialization. We record the full batch (LS)-gradient on each point along the descent path. Then we compute the (LS)-stochastic gradients on each points along the path by using different batch sizes and smoothing parameters $\sigma$. In computing (LS)-stochastic gradients we run 100 independent experiments. Then we compute the variance of the (LS)-stochastic gradient among these 100 experiments and regarding the full batch (LS)-gradient as the mean on each point along the full batch (LS)-GD descent path. For each pair of batch size and $\sigma$, we report the maximum variance over all the coordinates of the gradient and all the points along the descent path. We list the variance results in Table~\ref{Table:Variance} (note the case $\sigma=0$ corresponds to the SGD). These results show that compared to the SGD, LSGD with $\sigma=3$ can reduce the maximum variance  $\sim {\bf 100}$ times for different batch sizes. It is worth noting that the high order smoothing reduces more variance than the lower order smoothing, this might due to the fact that the noise of SGD is not Gaussian.

\begin{table*}[h]
\centering
\fontsize{10.0}{10}\selectfont
\begin{threeparttable}
\caption{The maximum variance of the stochastic gradient generated by LS-SGD with different $\sigma$ and batch size. $\sigma=0$ recovers the SGD.}
\begin{tabular}{cccccc}
\toprule
Batch Size     &\ \ \ \ \ \ \ \ \  2\ \ \ \ \ \ \ \   &\ \ \ \ \ \ \ \ 5\ \ \ \ \ \ \ \   &\ \ \ \ \ \ \ \ 10\ \ \ \ \ \ \ \  &\ \ \ \ \ \ \ \ 20\ \ \ \ \ \ \ \  &\ \ \ \ \ \ \ \ 50\ \ \ \ \ \ \ \  \cr
\midrule
$\sigma = 0$   & 1.50E-1  &5.49E-2  &2.37E-2 &1.01E-2 &4.40E-3 \cr
$\sigma = 1$   & 3.40E-3  &1.30E-3  &5.45E-4 &2.32E-4 &9.02E-5 \cr
$\sigma = 2$   & 2.00E-3  &7.17E-4  &3.46E-4 &1.57E-4 &5.46E-5 \cr
$\sigma = 3$   & 1.40E-3  &4.98E-4  &2.56E-4 &1.17E-4 &3.97E-5 \cr
\bottomrule
\end{tabular}
\label{Table:Variance}
\end{threeparttable}
\end{table*}

\section{Numerical Results on Avoid Local Minima and Speed Up Convergence}\label{section-numerics1}
We first show that LS-GD can bypass sharp minima and reach the global minima.
%for certain examples. 
We consider the following function, in which we `drill' narrow holes on a smooth convex function,
\begin{eqnarray}\label{Function-3D-Test}
f(x, y, z) = -4e^{-\left((x-\pi)^2+(y-\pi)^2+(z-\pi)^2\right)}-
\\ \nonumber
4\sum_{i}\cos(x)\cos(y)e^{-\beta\left((x-r\sin(\frac{i}{2})-\pi)^2+(y-r\cos(\frac{i}{2})-\pi)^2\right)},
\end{eqnarray}
where the summation is taken over the index set 
%$\{i| 0\leq \frac{i}{2}<2\pi\}$, $r$ and $\beta$ 
$\{i\in \mathbb{N}| \; 0 \leq i < 4\pi\}$, $r$ and $\beta$ 
are the parameters that determine the location and narrowness of the local minima and are set to $1$ and $\frac{1}{\sqrt{500}}$, respectively. We 
do GD and LS-GD
%do GD with regular and Laplacian smoothed gradient 
starting from a random point in the neighborhoods of the narrow minima, i.e., $(x_0, y_0, z_0)\in \{\bigcup_i U_\delta(r\sin(\frac{i}{2})+\pi, r\cos(\frac{i}{2})+\pi, \pi)| \; 0\leq i<4\pi, i\in \mathbb{N}\}$, where $U_\delta(P)$ is a neighborhood of the point $P$ with radius $\delta$. Our experiments (Fig. \ref{Demo1}) show that, if $\delta \leq 0.2$ GD will converge to a narrow local minima, while LS-GD convergences to the wider global minima.
\medskip

\begin{figure}[ht]
\centering
%\hskip -0.6cm
\begin{tabular}{cc}
\includegraphics[scale=0.45]{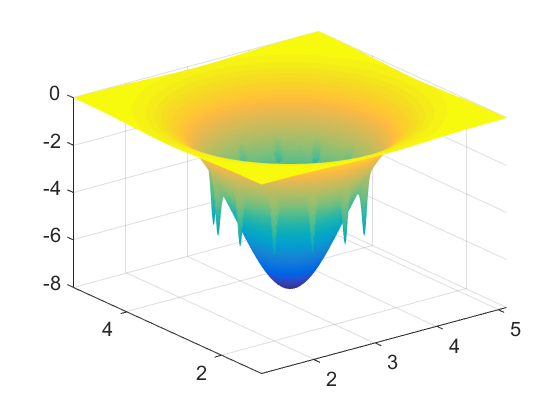} &
\includegraphics[scale=0.45]{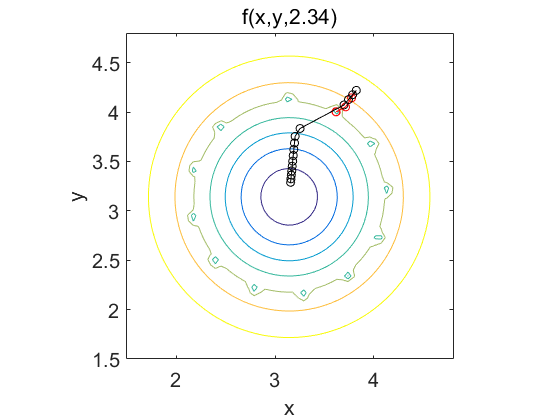}  \\
(a) & (b)\\
\end{tabular}
%\vskip -0.32cm
\caption{Demo of GD and LS-GD. Panel (a) depicts the slice of the function (Eq.(\ref{Function-3D-Test})) with $z=2.34$; panel (b) shows the paths of GD (red) and LS-GD (black). We take the step size to be 0.02 for both GD and LS-GD. $\sigma=1.0$ is utilized for LS-GD.}
\label{Demo1}
\end{figure}

Next, let us compare LSGD with some popular optimization methods on the benchmark 2D-Rosenbrock function which is a non-convex function. The global minimum is inside a long, narrow, parabolic shaped flag valley. To find the valley is trivial. To converge to the global minimum, however, is difficult. The function is defined by
\begin{equation}
\label{Rosenbrock}
f(x, y) = (a-x)^2 + b(y-x^2)^2,
\end{equation}
it has a global minimum at $(x, y) = (a, a^2)$, and we set $a=1$ and $b=100$ in the following experiments.
\medskip

Starting from the initial point with coordinate $(-3, -4)$, we run 2K iterations of the following optimizers including GD, GD with Nesterov momentum \cite{Nesterov:1983}, Adam \cite{Adam:2014}, RMSProp \cite{tieleman2012lecture}, and LSGD ($\sigma=0.5$). The step size used for all these methods is $3e-3$. Figure~\ref{fig:Loss-Rosenbrock} plots the iteration v.s. objective value, and it shows that GD together with Nesterov momentum converges faster than all the other algorithms. The second best algorithm is LSGD. Meanwhile, Nesterov momentum can be used to speed up LSGD, and we will show this numerically in training DNNs in section~\ref{section-deeplearning}.
\medskip

\begin{figure}[ht]
\centering
%\hskip -0.65cm
\begin{tabular}{c}
\includegraphics[scale=0.38]{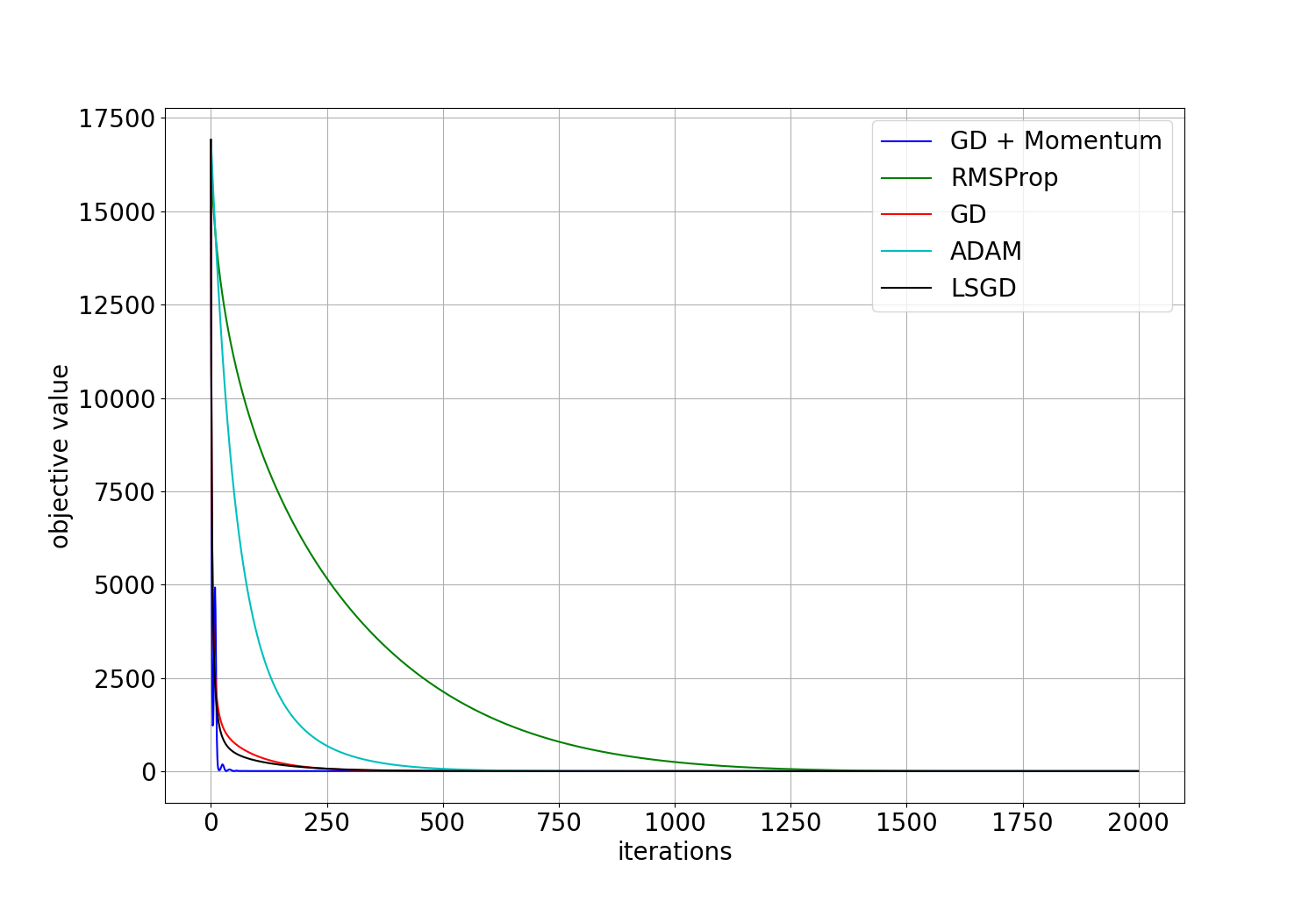}  \\
\end{tabular}
\caption{Iteration v.s. loss of different optimization algorithms in optimize the Rosenbrock function.}
\label{fig:Loss-Rosenbrock}
\end{figure}

Figure~\ref{fig:Snapshot} depicts some snapshots (The 300th, 600th, 900th, and 1200th iteration, respectively) of the trajectories of different optimization algorithms. These figures show that even though GD with momentum converge faster but it suffers from some overshoots, and they detour to converge to the local minima. All the other algorithms go along a direct path to the minima, and LSGD converges fastest.
\medskip

\begin{figure}[ht]
\centering
%\hskip -0.65cm
\begin{tabular}{cc}
\includegraphics[scale=0.18]{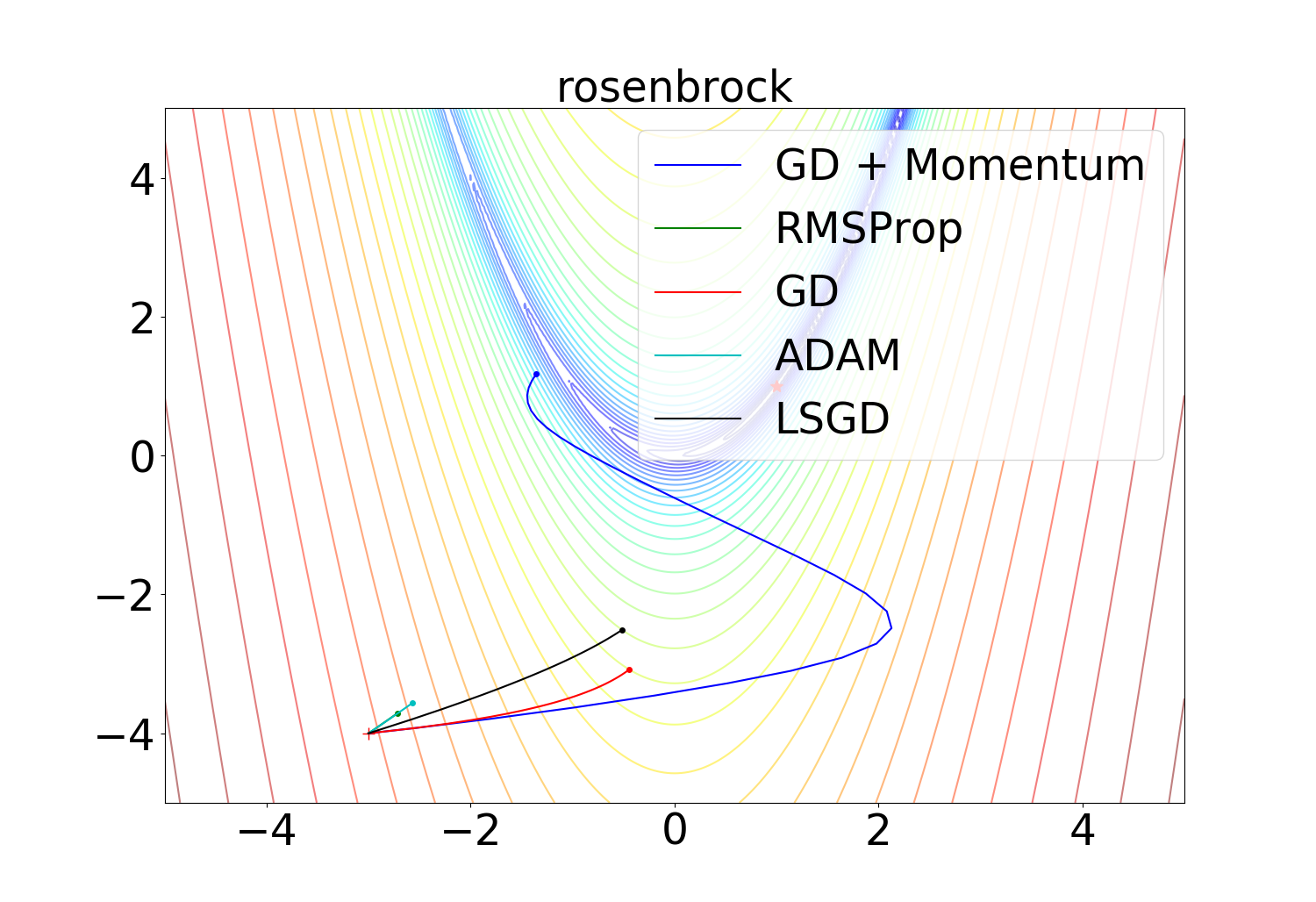} &
\includegraphics[scale=0.18]{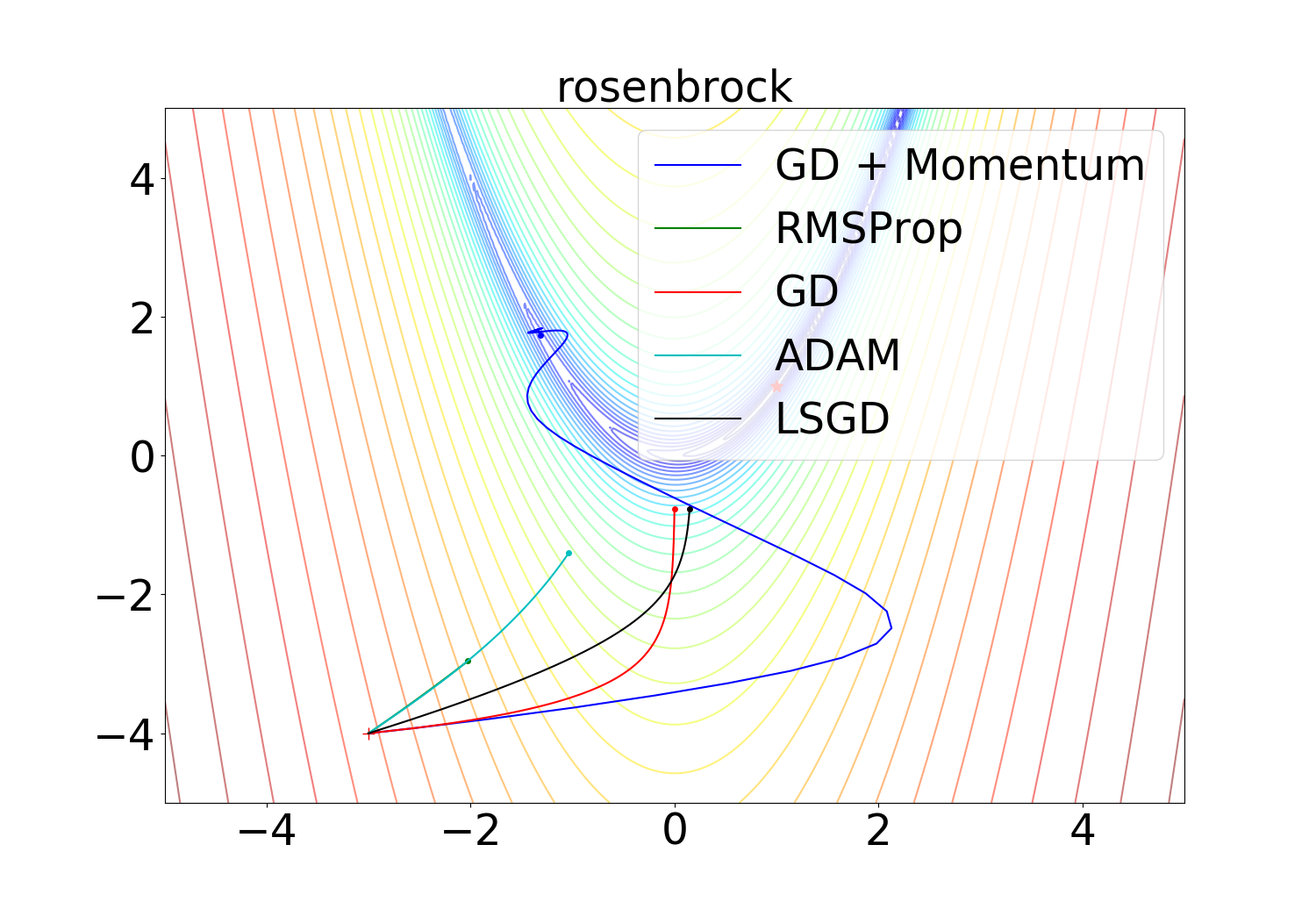} \\
Iteration: 300 & (b) Iteration: 600 \\
\includegraphics[scale=0.18]{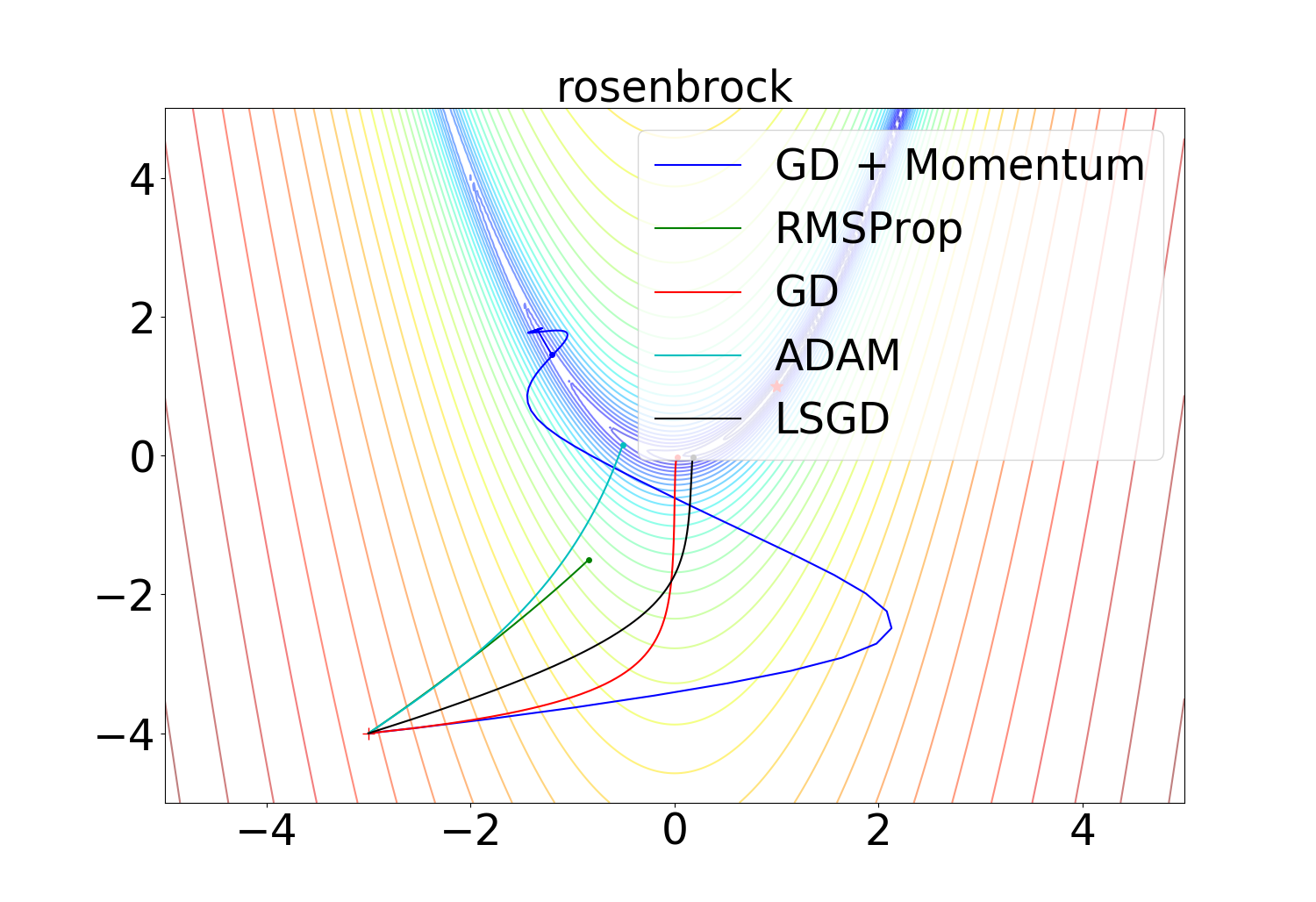} &
\includegraphics[scale=0.18]{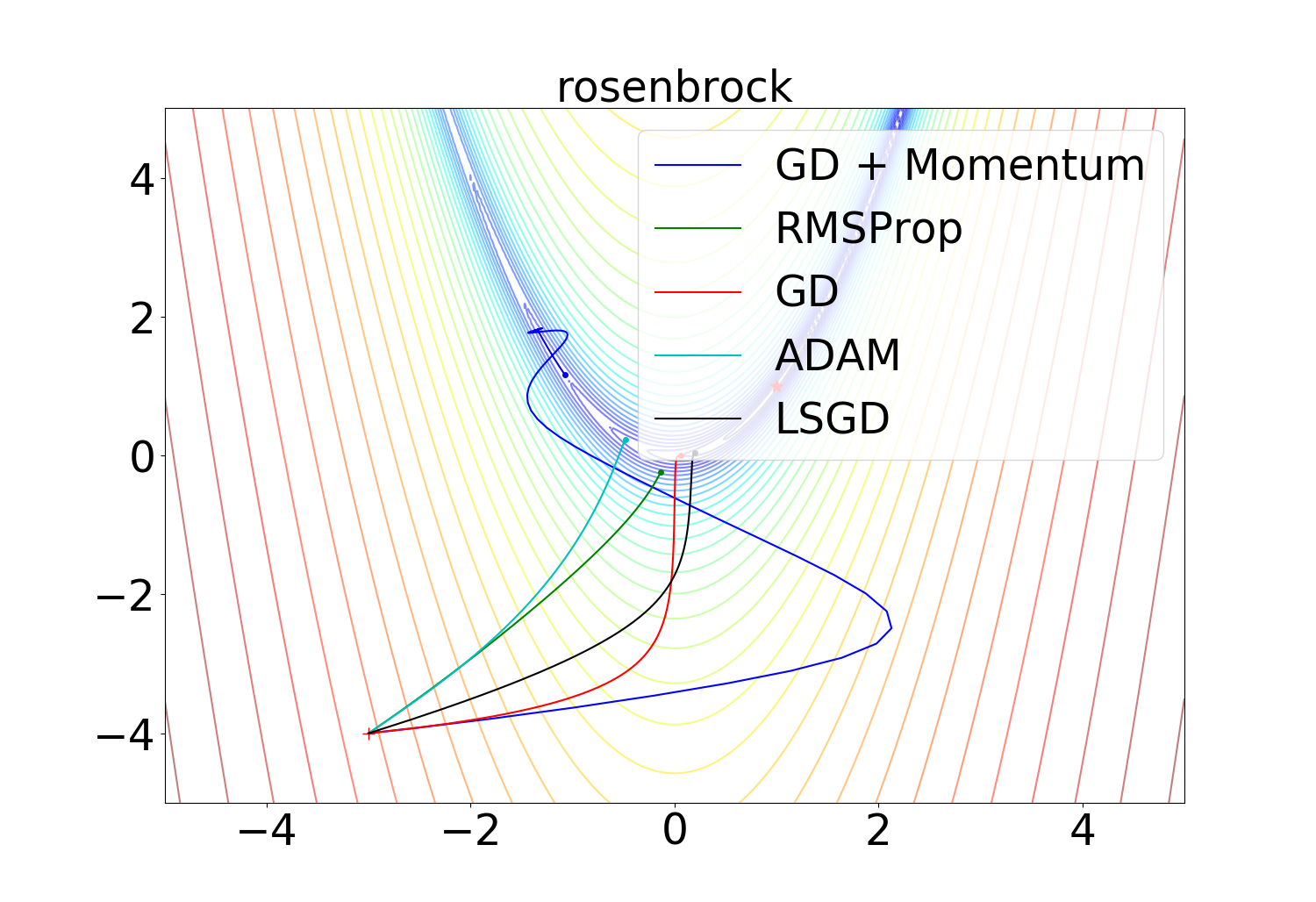}  \\
Iteration: 900 &(d) Iteration: 1200\\
\end{tabular}
\caption{Some snapshots of trajectories of different optimization algorithms on the Rosenbrock function.}
\label{fig:Snapshot}
\end{figure}

Furthermore, we will show that LSGD can be further accelerated by using Nesterov momentum. As show in Fig.~\ref{fig:Rosenbrock-HD}, the LSGD together with Nesterov momentum converges much faster than GD with momentum, especially for high dimensional Rosenbrock function.

\begin{figure}[!h]
\centering
%\hskip -0.65cm
\begin{tabular}{ccc}
\includegraphics[scale=0.18]{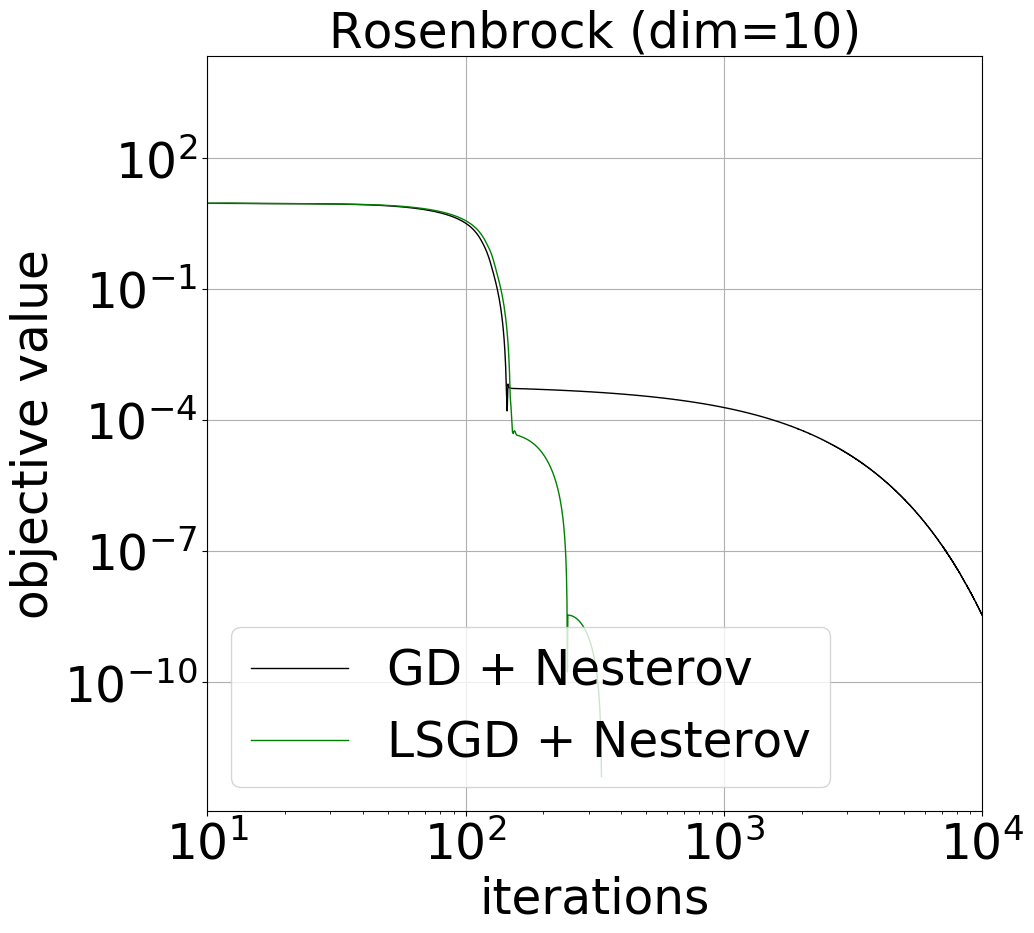} &
\includegraphics[scale=0.18]{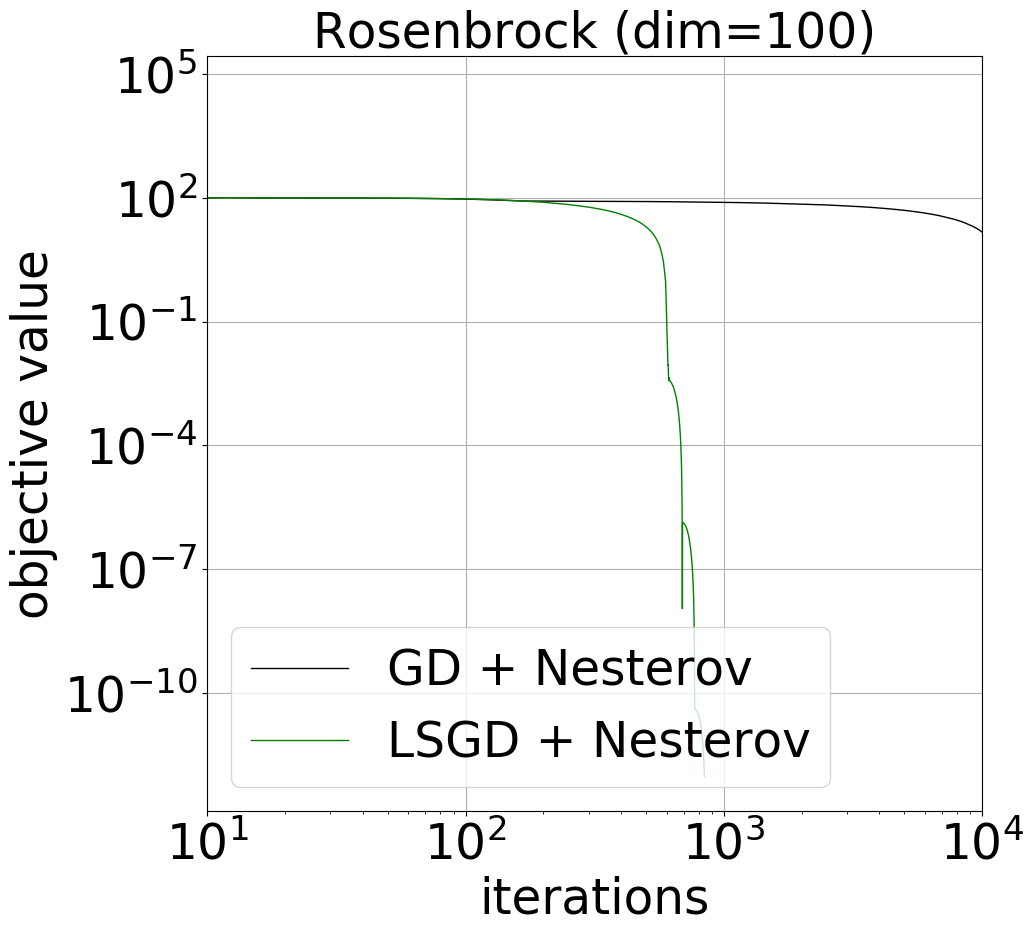} &
\includegraphics[scale=0.18]{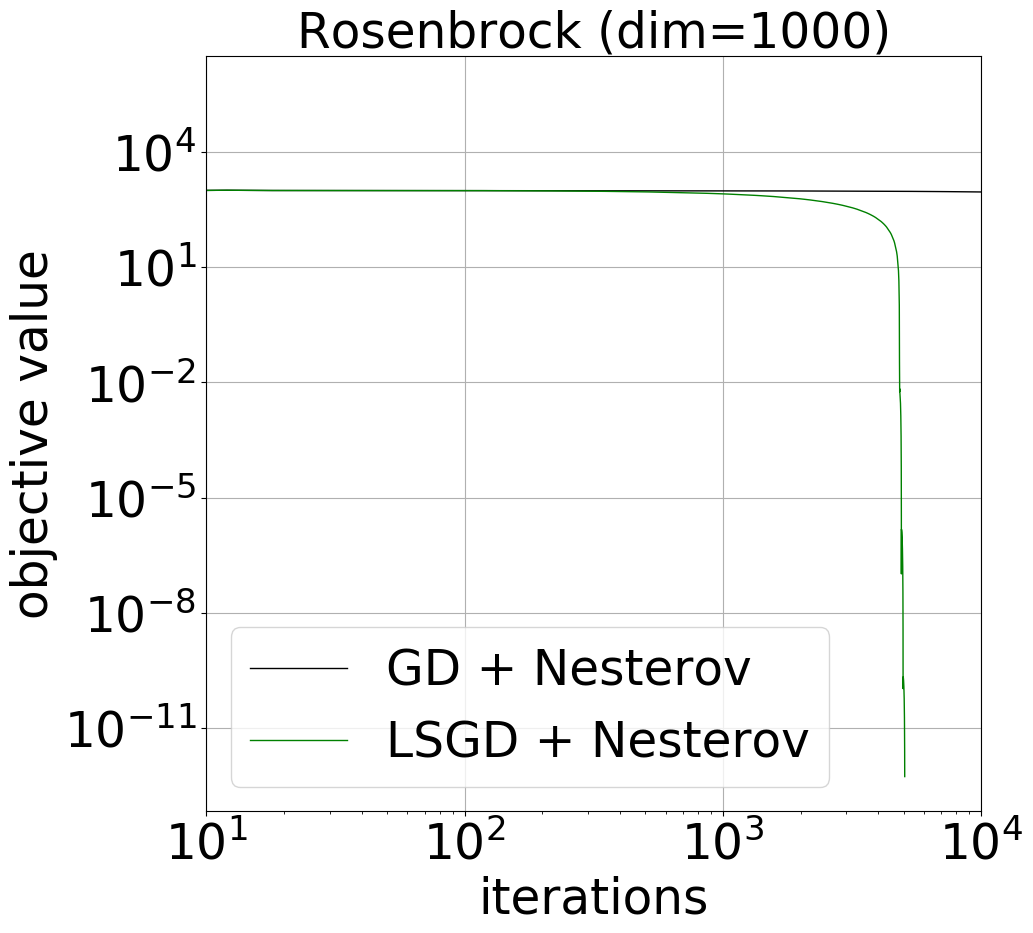}  \\
%10D & 100D & 1000D \\
\end{tabular}
\caption{Iteration v.s. objective value for GD with Nesterov momentum and LSGD with Nesterov momentum.}
\label{fig:Rosenbrock-HD}
\end{figure}

\section{Application to Deep Learning}\label{section-deeplearning}
\subsection{Train neural nets with small batch size}
Many advanced artificial intelligence tasks make high demands on training neural nets with extremely small batch sizes. The milestone technique for this is group normalization \cite{Wu:2018GroupNormalization}. In this section, we show that LS-SGD successfully trains DNN with extremely small batch size. We consider LeNet-5 \cite{LeCun:1998} for MNIST classification. Our network architecture is as follows
\begin{eqnarray*}
\mbox{LeNet-5:}\  {\rm input}_{28 \times 28} \rightarrow {\rm conv}_{20, 5, 2} \rightarrow
{\rm conv}_{50, 5, 2} \rightarrow {\rm fc}_{512} \rightarrow {\rm softmax}.
\end{eqnarray*}
The notation ${\rm conv}_{c, k, m}$ denotes a 2D convolutional layer with $c$ output channels, each of which is the sum of a channel-wise convolution operation on the input using a learnable kernel of size $k \times k$, it further adds ReLU nonlinearity and max pooling with stride size $m$. ${\rm fc}_{512}$ is an affine transformation that transforms the input to a vector of dimension 512. Finally, the tensors are activated by a multi-class Logistic function. The MNIST data is first passed to the layer ${\rm input}_{28 \times 28}$, and further processed by this hierarchical structure. We run $100$ epochs of both SGD and LS-SGD with initial learning rate $0.01$ and divide by $5$ after 50 epochs, and use a weight decay of $0.0001$ and momentum of $0.9$. Figure. \ref{fig:SmallBatchSize}(a) plots the generalization accuracy on the test set with the LeNet5 trained with different batch sizes. For each batch size, LS-SGD with $\sigma=1.0$ keeps the testing accuracy more than $99.4\%$, SGD reduce the accuracy to $97\%$ when batch size 4 is used. The classification become just a random guess, when the model is trained by SGD with batch size 2. Small batch size leads to large noise in the gradient, which may make the noisy gradient not along the decent direction; however, Lapacian smoothing rescues this by decreasing the noise.

%\begin{figure}[ht]
%\centering
%\begin{tabular}{c}
%\includegraphics[scale=0.1]{Accuracy_DifferentBatchSize.png} \\
%\end{tabular}
%\vskip -0.32cm
%\caption{Testing accuracy of LeNet5 trained by SGD/LS-SGD on MNIST with various batch %sizes.}\label{fig:SmallBatchSize}
%\end{figure}

\begin{figure}[ht]
\centering
\begin{tabular}{cc}
\includegraphics[scale=0.26]{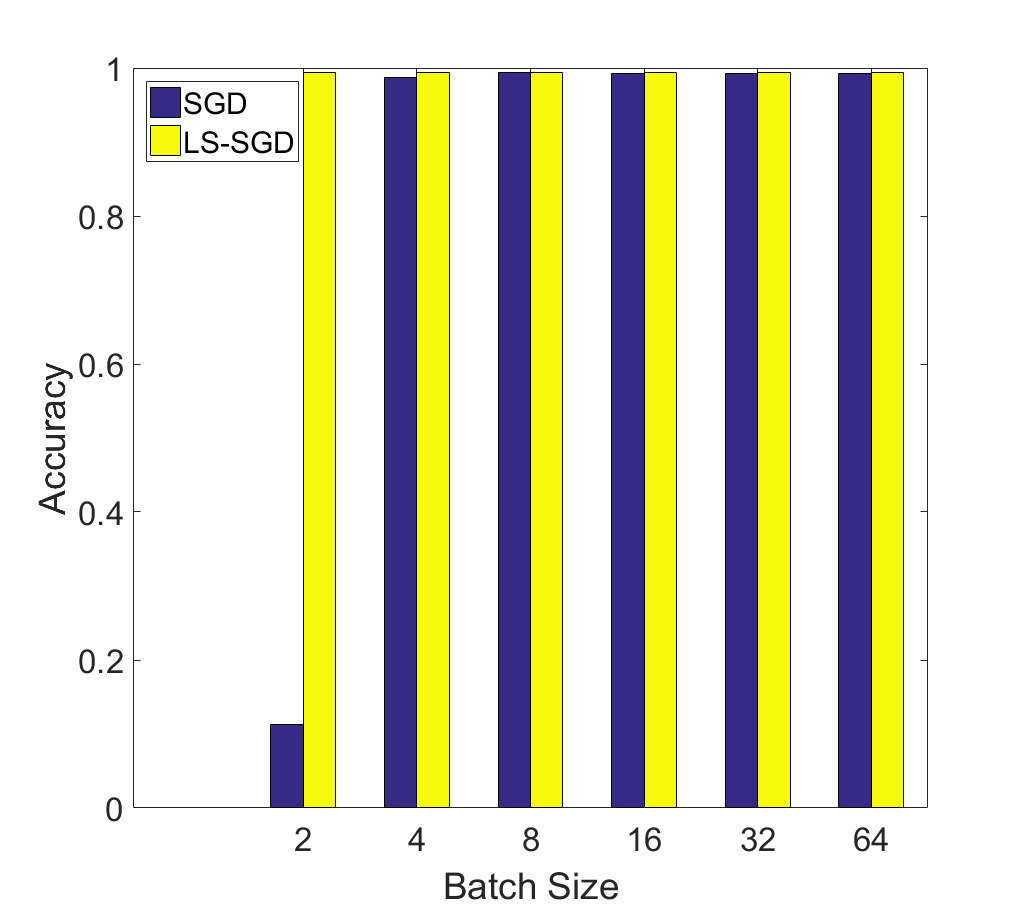}&
\includegraphics[scale=0.26]{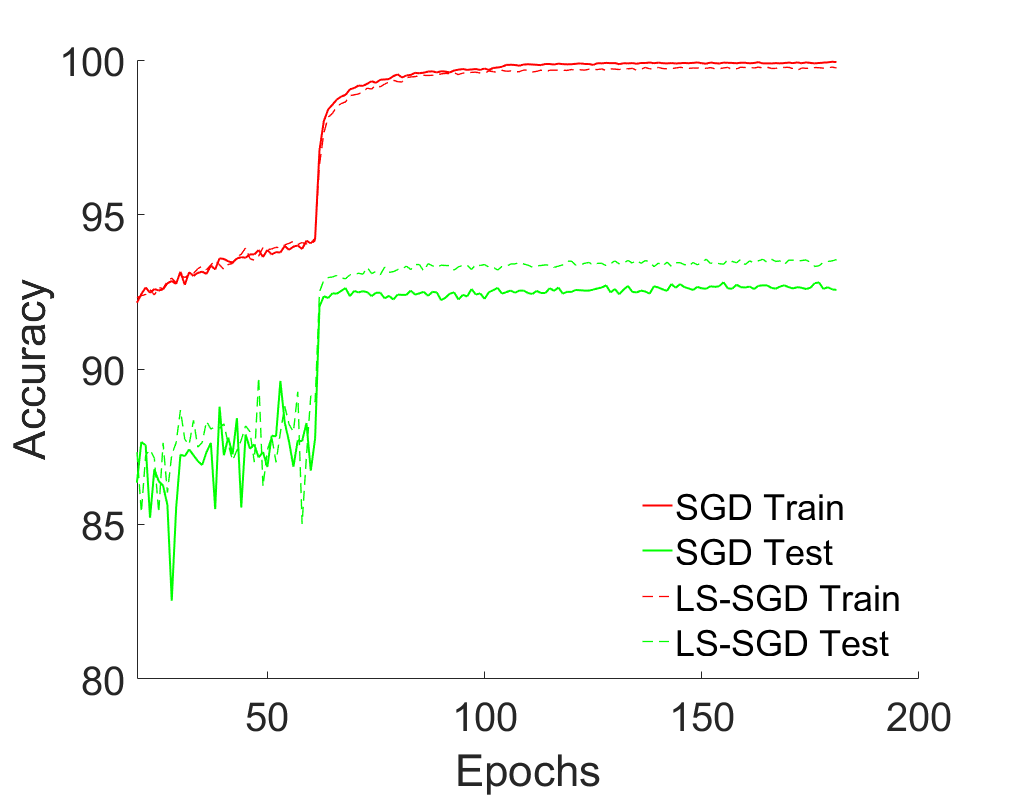}\\
\end{tabular}
%\vskip -0.32cm
\caption{(a). Testing accuracy of LeNet5 trained by SGD/LS-SGD on MNIST with various batch sizes. (b). The evolution of the pre-activated ResNet56's training and generalization accuracy by SGD and LS-SGD. (Start from the 20-th epoch.)}\label{fig:SmallBatchSize}
\end{figure}

\subsection{Improve generalization accuracy}
The skip connections in ResNet smooth the landscape of the loss function of the classical CNN \cite{ResNet,Li:2017}. This means that ResNet has fewer sharp minima. On Cifar10 \cite{Cifar:2009}, we compare the performance of LS-SGD and SGD on ResNet with the pre-activated ResNet56 as an illustration. We take the same training strategy as that used in \cite{ResNet}, except that we run $200$ epochs with the learning rate decaying by a factor of $5$ after every 40 epochs. For ResNet, instead of applying LS-SGD for all epochs, we only use LS-SGD in the first 40 epochs, and the remaining training is carried out by SGD (this will save the extra computational cost due to LS, and we noticed that the performance is similar to the case when LS is used for the whole training process). The parameter $\sigma$ is set to $1.0$. Figure~\ref{fig:SmallBatchSize}(b) depicts one path of the training and generalization accuracy of the neural nets trained by SGD and LS-SGD, respectively. It is seen that, even though the training accuracy obtained by SGD is higher than that by LS-SGD, the generalization is however inferior to that of LS-SGD. We conjecture that this is due to the fact that SGD gets trapped into some sharp but deeper minimum, which fits better than a flat minimum but generalizes worse. We carry out $25$ replicas of this experiments, the histograms of the corresponding accuracy are shown in Fig.~\ref{ResNet56-Training}.

%\begin{figure}[!h]
%\centering
%\vskip -0.32cm
%\begin{tabular}{c}
%\includegraphics[scale=0.1]{SGD-SSGD-3.png} \\
%\end{tabular}
%\vskip -0.32cm
%\caption{The evolution of the pre-activated ResNet56's training and generalization accuracies %by SGD and LS-SGD. (Start from the 20-th epoch.)}
%\label{SGD-SSGD}
%\end{figure}

\begin{figure}[!h]
\centering
\begin{tabular}{cc}
SGD & LS-SGD with $\sigma=1.0$  \\
\includegraphics[scale=0.45]{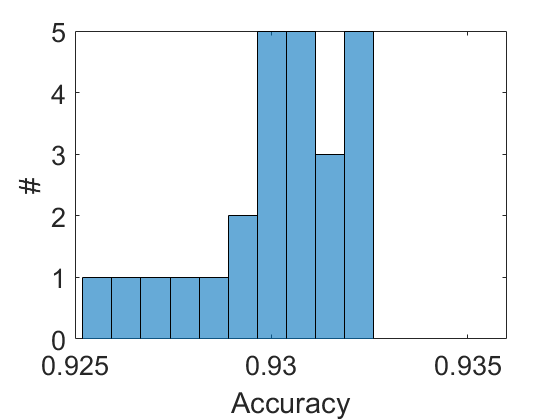} &
\includegraphics[scale=0.45]{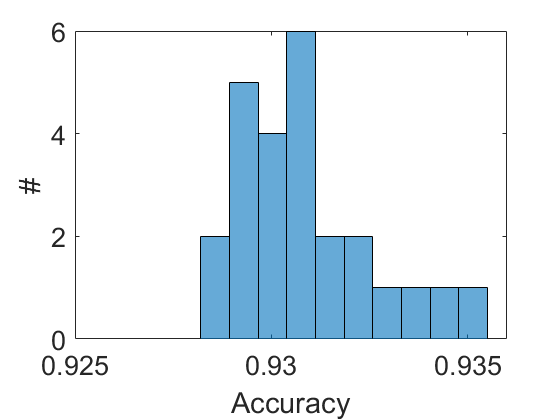} \\
\end{tabular}
%\vskip -0.32cm
\caption{The histogram of the generalization accuracy of the pre-activated ResNet56 on Cifar10 trained by SGD and LS-SGD over 25 independent experiments.}
\label{ResNet56-Training}
\end{figure}

\subsection{Training Wassersterin GAN}
Generative Adversarial Networks (GANs) \cite{GAN} are notoriously delicate and unstable to train \cite{arjovsky2017towards}. In \cite{Arjovsky:2017}, Wasserstein-GANs (WGANs) are introduced to combat the instability in the training GANs. In addition to being more robust in training parameters and network architecture, WGANs provide a reliable estimate of the Earth Mover (EM) metric which correlates well with the quality of the generated samples. Nonetheless, WGANs training becomes unstable with a large learning rate or when used with a momentum based optimizer \cite{Arjovsky:2017}. In this section, we demonstrate that the gradient smoothing technique in this paper alleviates the instability in the training, and improves the quality of generated samples. Since WGANs with weight clipping are typically trained with RMSProp \cite{tieleman2012lecture}, we propose replacing the gradient $g$ by a smoothed version $g_\sigma = \A_\sigma^{-1}g$, and also update the running averages using $g_\sigma$ instead of $g$. We name this algorithm LS-RMSProp. 
\medskip

To accentuate the instability in training and demonstrate the effects of gradient smoothing, we deliberately use a large learning rate for training the generator. We compare the regular RMSProp with the %proposed 
LS-RMSProp. The learning rate for the critic is kept small and trained approximately to convergence so that the critic loss is still an effective approximation to the Wasserstein distance. To control the number of unknowns in the experiment and make a meaningful comparison using the critic loss, we use the classical RMSProp for the critic, and only apply LS-RMSProp to the generator. 
\medskip

\begin{figure}[!h]
\centering
\begin{tabular}{c}
RMSProp \\
\includegraphics[width=90mm]{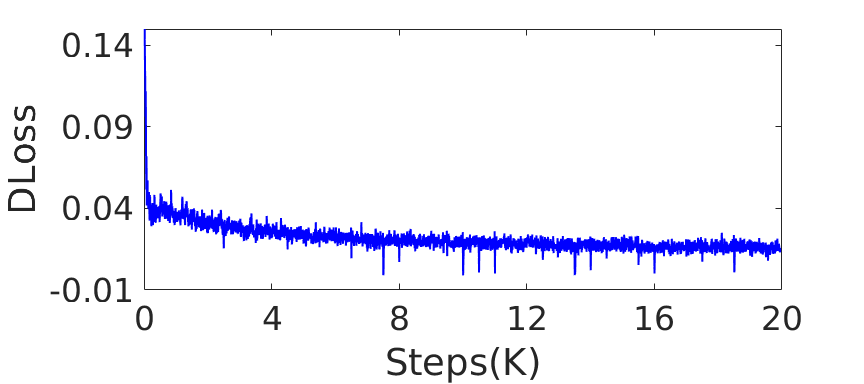} \\
LS-RMSProp, $\sigma=3.0$  \\
\includegraphics[width=90mm]{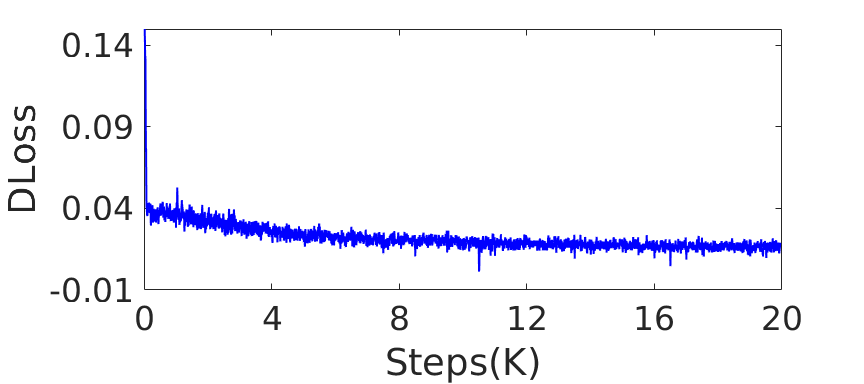} \\
\end{tabular}
%\vskip -0.32cm
\caption{Critic loss with learning rate $lrD = 0.0001$, $lrG = 0.005$ for RMSProp (top) and LS-RMSProp (bottom), trained for 20K iterations. We apply a mean filter of window size 13 for better visualization. The loss from LS-RMSProp is visibly less noisy.}
\label{fig:wgan-largelr-loss}
\end{figure}

\begin{figure}[!h]
\centering
\begin{tabular}{cc}
RMSProp & LS-RMSProp, $\sigma=3.0$ \\
\includegraphics[width=50mm]{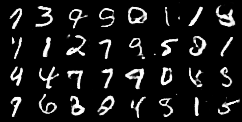} &
\includegraphics[width=50mm]{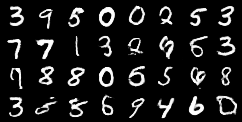} \\
(a) & (b)\\
RMSProp & LS-RMSProp, $\sigma=3.0$\\
\includegraphics[width=50mm]{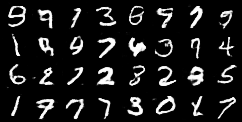} &
\includegraphics[width=50mm]{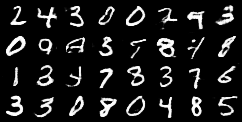}\\
(c) & (d)\\
\end{tabular}
%\vskip -0.32cm
\caption{Samples from WGANs trained with RMSProp (a, c) and LS-RMSProp (b, d).  The learning rate is set  to $lrD = 0.0001$, $lrG = 0.005$ for both RMSProp and LS-RMSProp in (a) and (b). And $lrD = 0.0001$, $lrG = 0.0001$ are used for  both RMSProp and LS-RMSProp in (c) and (d). The critic is trained for 5 iterations per step of the generator, and 200 iterations per every 500 steps of the generator.}
\label{fig:wgan-largelr-samples}
\end{figure}

We train the WGANs on the MNIST dataset using the DCGAN \cite{radford2015unsupervised} for both the critic and generator. In Figure \ref{fig:wgan-largelr-loss} (top),  we observe the loss for RMSProp trained with a large learning rate has multiple sharp spikes, indicating instability in the training process. The samples generated are also lower in quality, containing noisy spots as shown in Figure \ref{fig:wgan-largelr-samples} (a). In contrast, the curve of training loss for LS-RMSProp is smoother and exhibits fewer spikes. The generated samples as shown in Fig.~\ref{fig:wgan-largelr-samples} (b) are also of better quality and visibly less noisy. The generated characters shown in Fig.~\ref{fig:wgan-largelr-samples} (b) are more realistic compared to the ones shown in Fig.~\ref{fig:wgan-largelr-samples} (a).
%A reasonable conjecture for this improved stability is that the spikes which appear in training for RMSProp correspond to sharp extrema in the approximate EM metric, and this is circumvented by smoothing the gradient. 
The effects are less pronounced with a small learning rate, but still result in a modest improvement in sample quality as shown in Figure \ref{fig:wgan-largelr-samples} (c) and (d).%\ref{fig:wgan-smalllr-samples}.
We also apply LS-RMSProp for training the critic, but do not see a clear improvement in the quality. This may be because the critic is already trained near optimality during each iteration, and does not benefit much from gradient smoothing.

\begin{figure}[!h]
\centering
\begin{tabular}{cc}
RMSProp & LS-RMSProp, $sigma=1.0$  \\
\includegraphics[width=72mm]{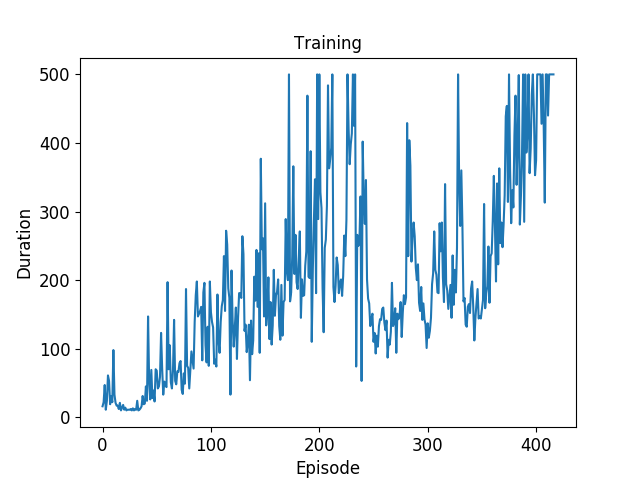} &
\includegraphics[width=72mm]{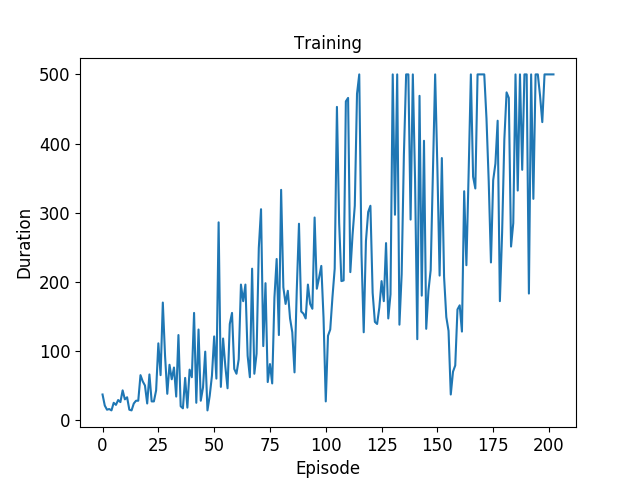} \\
\end{tabular}
%\vskip -0.32cm
\caption{Durations of the cartpole game in the training procedure. Left and right are training procedure by RMSProp and LS-RMSProp with $\sigma=1.0$, respectively.}
\label{fig:DRL}
\end{figure}

\subsection{Deep reinforcement learning}
Deep reinforcement learning (DRL) has been applied to playing games including Cartpole \cite{OpenAIGym:2016}, Atari \cite{Atari:2013}, Go \cite{AlphaGo:2016,Mnih:2015}. DNN plays a vital role in approximating the Q-function or policy function. We apply the Laplacian smoothed gradient to train the policy function to play the Cartpole game. We apply the standard procedure to train the policy function by using the policy gradient \cite{OpenAIGym:2016}. And we use the following network to approximate the policy function:
$$
{\rm input}_4 \rightarrow {\rm fc}_{20} \rightarrow {\rm relu} \rightarrow {\rm fc}_2 \rightarrow {\rm softmax}.
$$
The network is trained by RMSProp and LS-RMSProp with $\sigma=1.0$, respectively. The learning rate and other related parameters are set to be the default ones in PyTorch. The training is stopped once the average duration of 5 consecutive episodes is more than 490. In each training episode, we set the maximal steps to be 500. Left and right panels of Fig. \ref{fig:DRL} depict a training procedure by using RMSProp and LS-RMSProp, respectively. We see that Laplacian smoothed gradient takes fewer episodes to reach the stopping criterion. Moreover, we run the above experiments 5 times independently, and apply the trained model to play Cartpole. The game lasts more than 1000 steps for all the 5 models trained by LS-RMSProp, while only 3 of them lasts more than 1000 steps when the model is trained by vanilla RMSProp.

\section{Convergence Analysis}\label{section-convergence-analysis}
Note that the LS matrix $\A_\sigma^{-1}$ is positive definite and its largest and smallest eigenvalues are 1 and $\frac{1}{1+4\sigma}$, respectively. It is straightforward to show that all the convergence results for (S)GD still hold for LS(S)GD. In this section, we will show some additional convergence for LS(S)GD with a focus on LSGD, the corresponding results for LSSGD follow in a similar way.

\begin{proposition}\label{Hp-Convergence}
Consider the algorithm $\w^{k+1} = \w^k - \eta_k (\A_\sigma^n)^{-1} \nabla f(\w^k)$. Suppose $f$ is $L$-Lipschitz smooth and $0< \tilde{\eta} \leq \eta\leq \bar{\eta}< \frac{2}{L}$. Then $\lim_{t\to\infty} \|\nabla f(\w^k)\|\to 0$. Moreover, if the Hessian $\nabla^2 f$ of $f$ is continuous with $\w^*$ being the minimizer of $f$, and  $\bar{\eta}\|\nabla^2 f\|<1$,
then $\|\w^k-\w^*\|_{\A_\sigma^n}\to 0$ as $k\to \infty$, and the convergence is linear.
\end{proposition}

\begin{proof}%[\bf Proof of Proposition \ref{prop:noncvx}]
By the Lipschitz continuity of $\nabla f$ and the descent lemma \cite{bertsekas1999nonlinear}, we have
\begin{align*}
    f(\w^{k+1}) & \; = f(\w^k - \eta_k(\A_\sigma^n)^{-1}\nabla f(\w^k)) \\
    & \; \leq f(\w^k) - \eta_k \langle \nabla f(\w^k), (\A_\sigma^n)^{-1}\nabla f(\w^k))\rangle + \frac{\eta^2_k L}{2} \|(\A_\sigma^n)^{-1} \nabla f(\w^k)\|^2 \\
     & \; \leq f(\w^k) - \eta_k \| \nabla f(\w^k)\|_{(\A_\sigma^n)^{-1}}^2  + \frac{\eta^2_k L}{2} \| \nabla f(\w^k)\|_{(\A_\sigma^n)^{-1}}^2 \\
     & \; \leq f(\w^k) - \tilde{\eta}\left(1-\frac{\bar{\eta} L}{2}\right) \| \nabla f(\w^k)\|_{(\A_\sigma^n)^{-1}}^2.
\end{align*}
Summing the above inequality over $k$, we have
$$
\tilde{\eta}\left(1-\frac{\bar{\eta} L}{2}\right) \sum_{k=0}^{\infty} \| \nabla f(\w^k)\|_{(\A_\sigma^n)^{-1}}^2 \leq f(\w^0) -\lim_{k\to\infty} f(\w^k)< \infty.
$$
Therefore, $ \| \nabla f(\w^k)\|_{(\A_\sigma^n)^{-1}}^2\to 0$, and thus $ \| \nabla f(\w^k)\|\to 0$.

For the second claim, we have 
\begin{align*}
&\w^{k+1} -\w^* \\
= & \; \w^k - \w^* -\eta_k (\A_\sigma^n)^{-1}(\nabla f(\w^k) - \nabla f(\w^*)) \\
= & \; \w^k - \w^* -\eta_k (\A_\sigma^n)^{-1} \left(\int_0^1 \nabla^2 f(\w^* + \tau (\w^{k+1} - \w^*))\cdot (\w^k - \w^*) \mathrm{d}\tau \right) \\
= & \; \w^k - \w^* -\eta_k (\A_\sigma^n)^{-1} \left(\int_0^1 \nabla^2 f(\w^* + \tau (\w^{k+1} - \w^*))\mathrm{d}\tau \cdot (\w^k - \w^*)\right) \\
= & \; (\A_\sigma^n)^{-\frac{1}{2}} \left(\I - \eta_k (\A_\sigma^n)^{-\frac{1}{2}} \int_0^1 \nabla^2 f(\w^* + \tau (\w^{k+1} - \w^*))\mathrm{d}\tau  (\A_\sigma^n)^{-\frac{1}{2}})\right) (\A_\sigma^n)^{\frac{1}{2}} (\w^k - \w^*)
\end{align*}

Therefore, 
$$
\|\w^{k+1} - \w^*\|_{\A_\sigma^n} \leq \left\|\I - \eta_t  (\A_\sigma^n)^{-\frac{1}{2}} \int_0^1 \nabla^2 f(\w^* + \tau (\w^{k+1} - \w^*))\mathrm{d}\tau  (\A_\sigma^n)^{-\frac{1}{2}} \right\|\|\w^k - \w^*\|_{\A_\sigma^n}.
$$
So if $\eta_k \|\nabla^2 f\| \leq \frac{1}{\|(\A_\sigma^n)^{-1}\|} = 1$, the result follows.
\end{proof}

\begin{remark}
The convergence result in Proposition~\ref{Hp-Convergence} is also call $H_\sigma^n$-convergence. This is because $\lla \vu, \A_\sigma^n \vu\rra = \|\vu\|^2 + \sigma\|\mD_+^n \vu\|^2 = \|\vu\|^2_{H_\sigma^n}$.
\end{remark}

\section{Discussion and Conclusion}\label{section-discussion-conclusion}
\subsection{Some more properties of Laplacian smoothing}
In Theorem~\ref{High-Prop-Estimate-L2}, we established a high probability estimate of the LS operator in reducing the $\ell_2$ norm of any given vector. The $\ell_1$ type of high probability estimation can be established in the same way. These estimates will be helpful to develop privacy-preserving optimization algorithms to train ML models that improve the utility of the trained models without sacrifice the privacy guarantee \cite{Wang:DPLSSGD:2019}.
\medskip

Regarding the $\ell_1$/$\ell_2$ estimates of the LS operator, we further have the following results.

\begin{customprop}{8}
Given vectors $\g$ and $\d = \A_\sigma^{-1}\g$, for any $p \in\mathbb{N}$, it holds that 
$
\|\D^p\d\|_1 \leq \|\D^p\g\|_1. %+ \frac{\|\D \g\|^2}{\|\g\|}.
$
The inequality is strict unless $\D^p \g$ is a constant vector. 
\end{customprop}

\begin{proof}%[\bf Proof of Proposition \ref{prop:tv}]
Observe that $\A_\sigma$ and $\mD_+$ commute; therefore, for any $p\in \mathbb{N}$,
$\A_\sigma(\mD_+^p\vd) = \mD_+^p\vg$. Thus
%Since $(1+2\sigma)d_i = g_i + \sigma d_{i+1} + \sigma d_{i-1}$, for any $p\in\mathbb{N}$, 
we have
$$
(1+2\sigma)(\D^p \d)_i = (\D^p \g)_i + \sigma (\D^p \d)_{i+1} + \sigma (\D^p\d)_{i-1}.
$$
So
$$
(1+2\sigma)|(\D^p \d)_i| \leq |(\D^p \g)_i| + \sigma |(\D^p \d)_{i+1}| + \sigma |(\D^p\d)_{i-1}|.
$$
The inequality is strict if there are sign changes among the $(\D^p \d)_{i-1}$, $(\D^p \d)_{i}$, $(\D^p \d)_{i+1}$.
Summing over $i$ and using periodicity, we have
$$
(1+2\sigma)\sum_{i=1}^m|(\D^p \d)_i| \leq \sum_{i=1}^m|(\D^p \g)_i| + 2\sigma \sum_{i=1}^m|(\D^p \d)_{i}|,
$$
and the result follows. The inequality is strict unless $\D^p \g$ is a constant vector.
\end{proof}

\begin{proposition}\label{thm:l2norm}
Given any vector $\vg\in \R^m$ and $\vd=(\A_\sigma^n)^{-1} \vg$, then
\begin{equation}\label{eqn:l2_reduction_formula}
\|\vg\|^2 = \|\vd\|^2 + 2\sigma \|\mD_+^\vn \vd\|^2 + \sigma^2 \|\mL^n \vd\|^2,
\end{equation}
the variance of $\vd$ is much less than that of $\vg$.
\end{proposition}

\begin{proof}Observe that $\vg = \A_\sigma^n \vd = \vd + (-1)^n\sigma \mL^n d$. Therefore,
\begin{equation}\label{eqn:gsquared}
\|\vg\|^2 = \left \langle \vd + (-1)^n\sigma \mL^n \vd , \vd + (-1)^n\sigma \mL^n \vd \right\rangle = \|\vd\|^2 + 2 (-1)^n \sigma \langle \vd, \mL^n \vd\rangle + \sigma^2 \|\mL^n \vd\|^2.
\end{equation}
Next, note $\mD_-$ and $\mD_+$ are commute; thus
\begin{equation}\label{eqn:Ln}
\mL^n = \underbrace{(\mD_-\mD_+)\cdots (\mD_-\mD_+)}_{n} = \underbrace{\mD_-\cdots \mD_-}_{n} \underbrace{\mD_+\cdots \mD_+}_{n}  = \mD_-^n \mD_+^n.
\end{equation}
Now, we have
\begin{equation}\label{eqn:dLd}
\langle \vd, \mL^n \vd\rangle =
\langle \vd, \mD_-^n \mD_+^n d \rangle =
\langle (\mD_-^n)^T\vd,  \mD_+^n \vd \rangle =
\langle (-1)^n \mD_+^n \vd,  \mD_+^n \vd \rangle =
(-1)^n \|\mD_+^n \vd\|^2,
\end{equation}
where we used Eq.~(\ref{eqn:Ln}) in the first equality and $\mD_-=-\mD_+^T$ in the second to last equality.

Substituting Eq.~(\ref{eqn:dLd}) into Eq.~(\ref{eqn:gsquared}), yields Eq.~(\ref{eqn:l2_reduction_formula}).
\end{proof}

\subsection{Connection to Hamilton-Jacobi PDEs}
The motivation for the proposed LS-SGD comes from the Hamilton-Jacobi PDE (HJ-PDE). Consider the following unusual HJ-PDE with the empirical risk function, $f(\w),$ as initial condition
\begin{equation}
\label{HJ-PDE}
\begin{cases}
u_t + \frac{1}{2} \big\langle \nabla_\w u,  \A_\sigma^{-1}\nabla_\w u \big\rangle = 0, & (\w, t) \in \Omega \times [0, \infty)\\
u(\w, 0) = f(\w), & \w \in \Omega
\end{cases}
\end{equation}
By the Hopf-Lax formula \cite{Evans:2010}, the unique viscosity solution to Eq.~(\ref{HJ-PDE}) is represented by
$$
u(\w, t) = \inf_{\vv} \Big\{ f(\vv) +  \frac{1}{2t} \big\langle \vv-\w,\A_\sigma (\vv-\w) \big\rangle \Big\}.
$$

This viscosity solution $u(\w, t)$ makes $f(\w)$ "more convex", an intuitive definition and theoretical explanation of "more convex" can be found in \cite{Chaudhari:2017DeepRelaxation}, by bringing down the local maxima while retaining and widening local minima. An illustration of this is shown in Fig.~\ref{fig:tshirt}. If we perform the smoothing GD with proper step size on the function $u(\w,t)$, it is easier to reach the global or at least a flat minima of the original nonconvex function $f(\w)$. 

\begin{figure}[ht]
\centering
\begin{tabular}{c}
\includegraphics[scale=0.5]{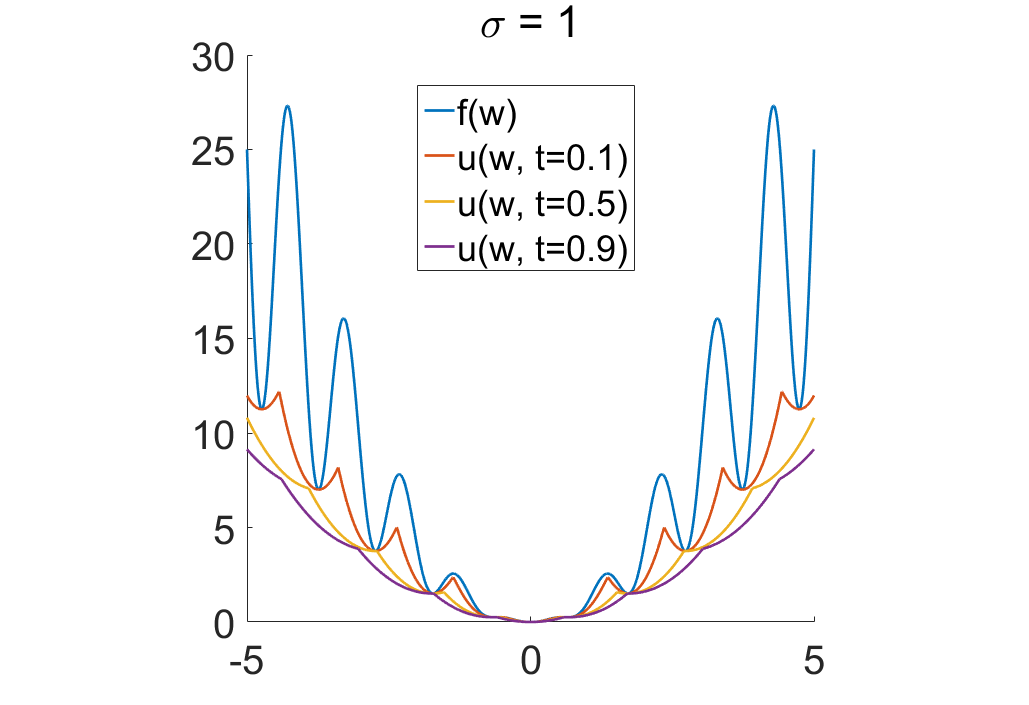} \\
\end{tabular}
%\vskip -0.32cm
\caption{$f(\w) = \|\w\|^2\big(1+\frac{1}{2} \sin(2\pi\|\w\|)\big)$ is made more convex by solving Eq.(\ref{HJ-PDE}). The plot shows the cross section of the 5D problem with $\sigma = 1$ and different $t$ values.}\label{fig:tshirt}
\end{figure}

\begin{prop}\label{prop:env}
Suppose $f(\w)$ is differentiable, the LS-GD on $u(\w,t)$
$$
\w^{k+1} = \w^k - t \A_\sigma^{-1}\nabla_\w u(\w^k,t)
$$
is equivalent to the smoothing implicit GD on $f(\w)$
\begin{equation}
\label{LS-Implicit}
\w^{k+1}  = \w^k - t \A_\sigma^{-1}\nabla f(\w^{k+1}).
\end{equation}
\end{prop}

\begin{proof}%[\bf Proof of Proposition \ref{prop:env}]
We define 
$$
z(\w,\vv,t):= f(\vv) + \frac{1}{2t}\langle \vv-\w, \A_\sigma(\vv-\w) \rangle, 
$$
and rewrite $u(\w,t) = \inf_{\vv} z(\w,\vv,t)$ as $z(\w,\vv(\w,t),t)$, where $\vv(\w,t) = \arg\min_{\vv} z(\w,\vv,t)$.
Then by the Euler-Lagrange equation,
\begin{align*}
\nabla_{\w} u(\w,t) = \nabla_{\w} z(\w,\vv(\w,t),t) = \mJ_\w\vv(\w,t) \nabla_{\vv} z(\w,\vv(\w,t),t) + \nabla_{\mathbf{w}} z(\w,\vv(\w,t),t),
\end{align*}
where $\mJ_\w\mathbf{v}(\w,t)$ is the Jacobian matrix of $\vv$ w.r.t. $\w$. Notice that $\nabla_{\vv} z(\w,\vv(\w,t),t) = \mathbf{0}$, 
$$
\nabla_{\w} u(\w,t) = \nabla_{\w} z(\w,\vv(\w,t),t) = -\frac{1}{t}\A_\sigma(\vv(\w,t)-\w).
$$
Letting $\w = \w^k$ and $\w^{k+1} = \vv(\w^k,t) = \arg\min_{\mathbf{v}} z(\w^k,\vv,t)$ in the above equalities, we have 
$$
\nabla_{\w} u(\w^k,t) = -\frac{1}{t}\A_\sigma(\w^{k+1}-\w^k).
$$ 
In summary, the gradient descent $\w^{k+1} = \w^k - t \A_{\sigma}^{-1} \nabla_{\w} u(\w^k,t)$ is equivalent to the proximal point iteration $\w^{k+1} = \arg\min_{\vv} f(\vv) + \frac{1}{2t}\langle \vv-\w^k, \A_\sigma(\vv -\w^k) \rangle$, which yields $\w^{k+1}  = \w^k - t \A_\sigma^{-1}\nabla f(\w^{k+1})$.
\end{proof}

The studied LS-GD algorithm is an explicit relaxation of the implicit algorithm in Eq.(\ref{LS-Implicit}).

\subsection{Conclusion}
Motivated by the theory of Hamilton-Jacobi partial differential equations, we proposed Laplacian smoothing gradient descent and its high order generalizations. This simple modification dramatically reduces the variance and optimality gap in stochastic gradient descent, allows us to take a larger step size, and helps to find better minima. Extensive numerical examples ranging from toy cases and shallow and deep neural nets to generative adversarial networks and deep reinforcement learning, all demonstrate the advantage of the proposed smoothed gradient. Several issues remain, in particular devising an on-the-fly adaptive method for choosing the smoothing parameter $\sigma$ instead of using a fixed value.

\section{Appendix}\label{section-appendix}

\subsection{Proof of Theorem~\ref{L2-Results-High-Prob}}
In this part, we will give a proof for Theorem~\ref{L2-Results-High-Prob}. 

\begin{lemma}\cite{TaoBlog}
\label{basic}
Let $t, u > 0$, $\vv$ be an $m$-dimensional standard normal random vector, and let $F:\mathbb{R}^m \rightarrow \mathbb{R}$ be a function such that $\|F(\vx) - F(\vy)\| \leq \|\vx - \vy\|$ for all $\vx$, $\vy \in \mathbb{R}^m$. Then
\begin{equation}
\label{Lemma-1-Eq}
\mathbb{P}\left( F(\vv) \geq \mathbb{E}F(\vv) + u \right) \leq \exp{\left(-tu+\frac{1}{2}\left(\frac{\pi t}{2}\right)^2\right)}.
\end{equation}
\end{lemma}

Taking $t=\frac{4}{\pi^2}$ in Lemma~\ref{basic}, we obtain

\begin{lemma}
\label{basic2}
Let $u > 0$, $\vv$ be an $m$-dimensional standard normal random vector, and let $F:\mathbb{R}^m \rightarrow \mathbb{R}$ be a function such that $\|F(\vx) - F(\vy)\| \leq \|\vx - \vy\|$ for all $\vx$, $\vy \in \mathbb{R}^m$. Then
\begin{equation}
\label{Lemma-2-Eq}
\mathbb{P}\left( F(\vv) \geq \mathbb{E}F(\vv) + u \right) \leq \exp{\left( -\frac{2}{\pi^2}u^2 \right)}.
\end{equation}
\end{lemma}

\begin{lemma}\label{ld}
Let $\vv$ be an $m$-dimensional standard normal random vector. Let $1\leq p\leq\infty$. Let $0< u <\mathbb{E}\|\vv\|_{\ell_{p}}$. Let $\mT\in\mathbb{R}^{m\times m}$ be such that $\|\mT\vx\|_{\ell_{p}}\leq\|\vx\|_{\ell_{p}}$ for all $\vx\in\mathbb{R}^{m}$. Then
\[\mathbb{P}\left(\|\mT\vv\|_{\ell_{p}}\geq
\frac{\mathbb{E}\|\mT\vv\|_{\ell_{p}}+u}{\mathbb{E}\|\vv\|_{\ell_{p}}-u}
\|\vv\|_{\ell_{p}}\right)\leq 2\exp{\left(-\frac{2}{\pi^{2}}u^{2}\right)}.\]
\end{lemma}

\begin{proof}
By Lemma \ref{basic2},
\[\mathbb{P}(\|\mT\vv\|_{\ell_{p}}\geq\mathbb{E}\|\mT\vv\|_{\ell_{p}}+u)\leq e^{-\frac{2}{\pi^{2}}u^{2}}\]
and
\[\mathbb{P}(-\|\vv\|_{\ell_{p}}\geq-\mathbb{E}\|\vv\|_{\ell_{p}}+u)\leq e^{-\frac{2}{\pi^{2}}u^{2}}.\]
The second inequality gives
\[\mathbb{P}(\|\vv\|_{\ell_{p}}\leq\mathbb{E}\|\vv\|_{\ell_{p}}-u)\leq e^{-\frac{2}{\pi^{2}}u^{2}}.\]
Therefore,
\begin{align*}
&\mathbb{P}\left(\|\mT\vv\|_{\ell_{p}}\geq\frac{\mathbb{E}\|\mT\vv\|_{\ell_{p}}+u}{
\mathbb{E}\|\vv\|_{\ell_{p}}-u}\|\vv\|_{\ell_{p}}\right)\\\leq&
\mathbb{P}(\|\mT\vv\|_{\ell_{p}}\geq\mathbb{E}\|\mT\vv\|_{\ell_{p}}+u)+
\mathbb{P}(\|\vv\|_{\ell_{p}}\leq\mathbb{E}\|\vv\|_{\ell_{p}}-u)\leq 2e^{-\frac{2}{\pi^{2}}u^{2}}.
\end{align*}
\end{proof}

\begin{lemma}\label{expectationlp}
Let $1\leq p\leq 2$. Let $\mT\in\mathbb{R}^{m\times m}$. Let $\vv$ be an $m$-dimensional standard normal random vector. Then
\[\mathbb{E}\|\mT\vv\|_{\ell_{p}}\leq m^{\frac{1}{p}-\frac{1}{2}}(\mathrm{Trace}\,\mT^{*}\mT)^{\frac{1}{2}}\left(\mathbb{E}|\vv_{1}|^{p}\right)^{\frac{1}{p}},\]
where $\vv_1$ is the first coordinate of $\vv$.
\end{lemma}
\begin{proof}
We write $\mT=(\mT_{i,j})_{1\leq i,j\leq n}$. Then
\begin{eqnarray*}
\mathbb{E}\|\mT\vv\|_{\ell_{p}}&=&\mathbb{E}\left(\sum_{i=1}^{n}\left|\sum_{j=1}^{n}\mT_{i,j}\vv_{j}\right|^{p}\right)^{\frac{1}{p}}\\&\leq&
\left(\sum_{i=1}^{n}\mathbb{E}\left|\sum_{j=1}^{n}\mT_{i,j}\vv_{j}\right|^{p}\right)^{\frac{1}{p}}\\&=&
\left(\sum_{i=1}^{n}\left(\sum_{j=1}^{n}\mT_{i,j}^{2}\right)^{\frac{p}{2}}\mathbb{E}|\vv_{1}|^{p}\right)^{\frac{1}{p}}\\&\leq&
\left(n^{1-\frac{p}{2}}\left(\sum_{1\leq i,j\leq n}\mT_{i,j}^{2}\right)^{\frac{p}{2}}\mathbb{E}|\vv_{1}|^{p}\right)^{\frac{1}{p}}\\&=&
n^{\frac{1}{p}-\frac{1}{2}}\left(\text{Trace }\mT^{*}\mT\right)^{\frac{1}{2}}\left(\mathbb{E}|\vv_{1}|^{p}\right)^{\frac{1}{p}},
\end{eqnarray*}
where the second equality follows from the assumption that $\vv$ is an $m$-dimensional standard normal random vector.
\end{proof}

\begin{lemma}\label{expectedl2}
Let $\vv$ be an $m$-dimensional standard normal random vector. Then
\[\mathbb{E}\|\vv\|_{\ell_{2}}\geq\sqrt{m}-\pi.\]
\end{lemma}
\begin{proof}
By Lemma \ref{basic2},
\[\mathbb{P}(\|\vv\|_{\ell_{2}}\geq\mathbb{E}\|\vv\|_{\ell_{2}}+u)\leq e^{-\frac{2}{\pi^{2}}u^{2}}\]
and
\[\mathbb{P}(-\|\vv\|_{\ell_{2}}\geq-\mathbb{E}\|\vv\|_{\ell_{2}}+u)\leq e^{-\frac{2}{\pi^{2}}u^{2}}.\]
Thus,
\[\mathbb{P}(|\|\vv\|_{\ell_{2}}-\mathbb{E}\|\vv\|_{\ell_{2}}|\geq u)\leq 2e^{-\frac{2}{\pi^{2}}u^{2}}.\]
Consider the random variable $W=\|\vv\|_{\ell_{2}}$. We have
\[\mathbb{E}|W-\mathbb{E}W|^{2}=\int_{0}^{\infty}\mathbb{P}(|W-\mathbb{E}W|\geq\sqrt{u})\,du\leq\int_{0}^{\infty}2e^{-\frac{2}{\pi^{2}}u}\,du=\pi^{2}.\]
Since $\mathbb{E}|W-\mathbb{E}W|^2 = \mathbb{E}W^2 - (\mathbb{E}W)^2$, we have
\[\mathbb{E}W\geq(\mathbb{E}W^{2})^{\frac{1}{2}}-(\mathbb{E}|W-\mathbb{E}W|^{2})^{\frac{1}{2}}\geq\sqrt{m}-\pi.\]
\end{proof}

\begin{lemma}\label{main-lemma}
Let $0<\epsilon<1-\frac{\pi}{\sqrt{m}}$. Let $\sigma>0$. Let
\[\beta=\frac{1}{m}\sum_{i=1}^{m}\frac{1}{1+2\sigma-\sigma z_{i}-\sigma\overline{z_{i}}},\]
where $z_{1},\ldots,z_{m}$ are the $m$ roots of unity. Let $\mB$ be the circular shift operator on $\mathbb{R}^{m}$. Let $\vv$ be an $m$-dimensional standard normal random vector. Then
\[\mathbb{P}\left(\|((1+2\sigma)\mI-\sigma \mB-\sigma \mB^{*})^{-1/2}\vv\|_{\ell_{2}}\geq\frac{\sqrt{\beta}+\epsilon}{1-\frac{\pi}{\sqrt{m}}-\epsilon}
\|\vv\|_{\ell_{2}}\right)\leq2e^{-\frac{2}{\pi^{2}}m\epsilon^{2}}.\]
\end{lemma}
\begin{proof}
Let $\mT=((1+2\sigma)\mI-\sigma \mB-\sigma \mB^{*})^{-1/2}$.
Taking $u=\sqrt{m}\epsilon$ in Lemma \ref{ld}, we have
\[\mathbb{P}\left(\|\mT\vv\|_{\ell_{2}}\geq
\frac{\mathbb{E}\|\mT\vv\|_{\ell_{2}}+\sqrt{m}\epsilon}{\mathbb{E}\|\vv\|_{\ell_{2}}-\sqrt{m}\epsilon}
\|\vv\|_{l^{2}}\right)\leq 2e^{-\frac{2}{\pi^{2}}m\epsilon^{2}}.\]
By Lemma \ref{expectationlp}, $\mathbb{E}\|\mT\vv\|_{\ell_{2}}\leq(\mathrm{Trace}\,\mT^{*}\mT)^{\frac{1}{2}}$. we have $\mathrm{Trace}\,\mT^{*}\mT=m\beta$. It is easy to show that $\mathrm{Trace}\,\mT^{*}\mT)=m\beta$ So $\mathbb{E}\|\mT\vv\|_{\ell_{2}}\leq\sqrt{m\beta}$. Also by Lemma \ref{expectedl2},
$\mathbb{E}\|\vv\|_{\ell_{2}}\geq\sqrt{m}-\pi$. Therefore,
\[\mathbb{P}\left(\|((1+2\sigma)\mI-\sigma \mB-\sigma \mB^{*})^{-1}\vv\|_{\ell_{2}}\geq\frac{\sqrt{\beta}+\epsilon}{1-\frac{\pi}{\sqrt{m}}-\epsilon}
\|\vv\|_{\ell_{2}}\right)\leq2e^{-\frac{2}{\pi^{2}}m\epsilon^{2}}.\]
\end{proof}

\begin{proof}[Proof of Theorem~\ref{L2-Results-High-Prob}]
Theorem~\ref{L2-Results-High-Prob} follows from Lemma~\ref{main-lemma} by substituting $\frac{\vv}{\|\vv\|_{\ell_2}}$ and using homogeneity and direct calculations.
\end{proof}

\subsection{Proof of Theorem~\ref{Theorem-Variance-Reduction}}
In this part, we will give a proof for Theorem~\ref{Theorem-Variance-Reduction}. 

\begin{lemma}[\cite{BIB:Bhatia97}]\label{lemma:eigenval_product}
Let $\prec_{w}$ denotes weak majorization. Denote eigenvalues of Hermitian matrix $\mX$, by $\lambda_1(\mX)\geq\ldots\geq \lambda_m(\mX)$. For every two Hermitian positive definite matrices $\mA$ and $\mB$, we have
$$
(\lambda_1(\A\B),\cdots,\lambda_m(\A\B)) \prec_w (\lambda_1(\A)\lambda_1(\B),\cdots,\lambda_m(\A)\lambda_m(\B)).
$$   
In particular, 
$$
\sum_{j=1}^{m} \lambda_j(\A\B) \leq \sum_{j=1}^{m}\lambda_j(\A)\lambda_j(\B).
$$
\end{lemma}

\begin{proof}[proof of Theorem~\ref{Theorem-Variance-Reduction}]
Let $\lambda_1\geq\ldots\geq \lambda_m$ denote the eigenvalues of $\Sigma$. The eigenvalues of $(A_\sigma^{n})^{-2}$ are given by $\{[1+4^n\sigma \sin^{2n}(\pi j/m)]^{-2}\}_{j=0}^{j=m-1}$, which we denote by $1=\alpha_1\geq\ldots\geq \alpha_m\geq (1+4^n\sigma)^{-2}$. 
We have
\begin{equation}\label{eqn:var_without_lsgd}
\sum_{j=1}^{m}\Var[\vn_j] = \trace(\Sigma) = \sum_{j=1}^{m} \lambda_j.
\end{equation}
On the other hand we also have
\begin{equation}\label{eqn:var_with_lsgd}
\sum_{j=1}^{m}\Var[(\A_\sigma^n)^{-1} \vn_j] = \trace((\A_\sigma^n)^{-1}\Sigma(\A_\sigma^n)^{-1}) = \trace((\A_\sigma^n)^{-2}\Sigma) \leq \sum_{j=1}^{m} \alpha_j \lambda_j, 
\end{equation}
where the last inequality is by lemma \ref{lemma:eigenval_product}. Now, 
\begin{align*}
\sum_{j=1}^{m} \lambda_j - \sum_{j=1}^{m} \alpha_j \lambda_j &= \sum_{j=1}^{m} (1-\alpha_j)\lambda_j  \\
&\geq \lambda_m (m - \sum_{j=1}^{m}  \alpha_j) \\
& = \frac{\lambda_1}{\kappa} (m - \sum_{j=1}^{m}  \alpha_j) \\
& \geq \frac{\sum_{j=1}^{m}\lambda_j}{m\kappa} (m - \sum_{j=1}^{m}  \alpha_j) \\ 
\end{align*}
Rearranging and simplifying above implies that
\begin{equation*}
\sum_{j=1}^{m}\alpha_j\lambda_j \leq (\sum_{j=1}^{m} \lambda_j)(1-\frac{1}{\kappa}+\frac{ \sum_{j=1}^{m}  \alpha_j}{m\kappa }).
\end{equation*} 
Substituting Eq.~(\ref{eqn:var_without_lsgd}) and Eq.~(\ref{eqn:var_with_lsgd}) in the above inequality, yields Eq.~(\ref{eqn:var_reduction_formula2}).
\end{proof}

\subsection{Proof of Lemma~\ref{thm:closed_formula}}
To proof Lemma~\ref{thm:closed_formula}, we first introduce the following lemma.

\begin{lemma}\label{lemma:DTFT}
For $0\leq \theta \leq 2\pi$, suppose 
$$
F(\theta) =  \frac{1}{1+2\sigma(1-\cos(\theta))},
$$
has the discrete-time Fourier transform of series $f[k]$. Then, for integer $k$,
$$
f[k] = \frac{\alpha^{|k|}}{\sqrt{4\sigma+1}}
$$
where 
$$
\alpha = \frac{2\sigma+1 - \sqrt{4\sigma+1}}{2\sigma}
$$
\end{lemma}
\begin{proof}
By definition, 
\begin{equation}\label{eqn:f_k_real}
f[k] = \frac{1}{2\pi} \int_{0}^{2\pi}  F(\theta) e^{ik\theta} \,d\theta=  \frac{1}{2\pi} \int_{0}^{2\pi}   \frac{e^{ik\theta}}{1+2\sigma(1-\cos(\theta))} \,d\theta.
\end{equation}
Computing Eq.~(\ref{eqn:f_k_real}) using Residue Theorem is a well-known technique in complex analysis. First, note that because $F(\theta)$ is real valued, $f[k]=f[-k]$; therefore, it suffices to compute Eq.~(\ref{eqn:f_k_real}) for nonnegative $k$. Set $z=e^{i\theta}$. Observe that $\cos(\theta)=0.5(z+1/z)$ and $dz=iz d\theta$. Substituting in Eq.~(\ref{eqn:f_k_real}) and simplifying yields that
\begin{equation}\label{eqn:f_k_complex}
f[k] = \frac{-1}{2\pi i \sigma}\oint \frac{z^k}{(z-\alpha_{-})(z-\alpha_{+}) } \,dz,
\end{equation}
where the integral is taken around the unit circle, and $\alpha_{\pm}= \frac{2\sigma+1 \pm \sqrt{4\sigma+1}}{2\sigma}$ are the roots of quadratic $-\sigma z^2 +(2\sigma+1)z -\sigma$. Note that $\alpha_{-}$ lies within the unit circle; whereas, $\alpha_{+}$ lies outside of the unit circle. Therefore, because $k$ is nonnegative, $\alpha_{-}$ is the only singularity of the integrand in Eq.~(\ref{eqn:f_k_complex}) within the unit circle. A straightforward application of the Residue Theorem yields that 
$$
f[k] = \frac{- \alpha_{-}^{k}}{\sigma (\alpha_{-}-\alpha_{+})} = \frac{\alpha^{k}}{\sqrt{4\sigma+1}}.
$$  
This completes the proof.
\end{proof}

Next, we give a proof for Lemma~\ref{thm:closed_formula}.
\begin{proof}[Proof of Lemma~\ref{thm:closed_formula}]
First observe that we can re-write the left hand side of Eq.~(\ref{eqn:closed_formula}) as
\begin{equation}
\frac{1}{m}\sum_{j=0}^{m-1} \frac{1}{1+2\sigma(1-\cos(\frac{2\pi j}{m}))}. 
\end{equation}
It remains to show that the above summation is equal to the right hand side of Eq.~(\ref{eqn:closed_formula}).
This follows by lemmas \ref{lemma:DTFT} and standard sampling results in Fourier analysis (i.e. sampling $\theta$ at points $\{2\pi j/m\}_{j=0}^{m-1}$). Nevertheless, we provide the details here for completeness: Observe that that the inverse discrete-time Fourier transform of
$$
G(\theta) = \sum_{j=0}^{m-1}\delta(\theta-\frac{2\pi j }{m}).
$$
is given by
$$
g[k] = 
\begin{cases}
m/2\pi \qquad &\text{if $k$ divides $m$,}\\
0 \qquad  &\text{otherwise.}
\end{cases}
$$   
Furthermore, let 
$$
F(\theta) =  \frac{1}{1+2\sigma(1-\cos(\theta))},
$$
and use $f[k]$ to denote its inverse discrete-time Fourier transform.
Now,
\begin{align*}
\frac{1}{m}\sum_{j=0}^{m-1} \frac{1}{1+2\sigma(1-\cos(\frac{2\pi j}{m}))} &= \frac{1}{m} \int_0^{2\pi} F(\theta)G(\theta) \\
&= \frac{2\pi}{m}  \DTFT^{-1}[F\cdot G][0]  \\
&= \frac{2\pi}{m}  (\DTFT^{-1}[F] * \DTFT^{-1}[G])[0] \\
&=  \frac{2\pi}{m}   \sum_{r=-\infty}^{\infty} f[-r]g[r] \\
&=  \frac{2\pi}{m}  \sum_{\ell=-\infty}^{\infty} f[-\ell m]  \frac{m}{2\pi} \\
&=  \sum_{\ell=-\infty}^{\infty} f[-\ell m].
\end{align*}
The proof is completed by substituting the result of lemma \ref{lemma:DTFT} in the above sum and simplifying.
\end{proof}

\clearpage
\section*{Acknowledgments}
This material is based on research sponsored by the Air Force Research Laboratory under grant numbers FA9550-18-0167 and MURI FA9550-18-1-0502, the Office of Naval Research under grant number N00014-18-1-2527, the U.S. Department of Energy under grant number DOE SC0013838, and by the National Science Foundation under grant number DMS-1554564, (STROBE). We would like to thank Jialin Liu and professors Pratik Chaudhari, Adam Oberman and Ming Yan for stimulating discussions.

%\bibliography{references.bib}
%\bibliographystyle{plain}

\end{document}